\documentclass[colorlinks=true, allcolors=blue]{fairmeta}

\usepackage{amsmath,amsfonts,bm}

\def\eqref#1{equation~\ref{#1}}

\def\1{\bm{1}}

\DeclareMathAlphabet{\mathsfit}{\encodingdefault}{\sfdefault}{m}{sl}
\SetMathAlphabet{\mathsfit}{bold}{\encodingdefault}{\sfdefault}{bx}{n}

\newcommand{\R}{\mathbb{R}}

\usepackage[utf8]{inputenc} 
\usepackage[T1]{fontenc}    
\usepackage{url}            
\usepackage{booktabs}       
\usepackage{amsfonts}       
\usepackage{nicefrac}       
\usepackage{microtype}      
\usepackage[table]{xcolor}         
\usepackage{comment}
\usepackage{csquotes}
\usepackage{makecell}

\usepackage{amsmath}
\usepackage{algorithm}
\usepackage{algpseudocode}
\usepackage{graphicx}
\usepackage{amsthm}
\usepackage{amssymb}
\usepackage{thmtools}
\usepackage{thm-restate}
\usepackage{subcaption}
\usepackage{wrapfig}

\usepackage{enumitem} 
\usepackage{empheq}
\usepackage[most]{tcolorbox} 

\DeclareMathOperator{\diloco}{diloco}
\DeclareMathOperator{\primal}{primal}
\DeclareMathOperator{\modern}{modern}
\DeclareMathOperator{\baseopt}{\texttt{BaseOpt}}
\DeclareMathOperator{\baseoptiter}{\texttt{BaseOptIteration}}

\colorlet{mygray}{gray!50}

\definecolor{colorblue}{HTML}{E7F3FF}
\definecolor{colorred}{HTML}{FDF2F0}
\definecolor{mygray}{gray}{0.6}
\definecolor{colordarkgray}{HTML}{E0E0E0}  
\definecolor{colorlightgray}{HTML}{FFFFFF} 
\definecolor{colorbrown}{HTML}{F8F0E3} 
\definecolor{colorgreen}{HTML}{E8F5E9} 

\newtheorem{thm}{Theorem}
\newtheorem{lem}{Lemma}
\newtheorem{prop}{Proposition}
\newtheorem{corollary}{Corollary}

\usepackage[colorinlistoftodos,bordercolor=orange,backgroundcolor=orange!20,linecolor=orange,textsize=scriptsize]{todonotes}

\title{Smoothing DiLoCo with Primal Averaging for Faster Training of LLMs}

\author[1]{Aaron Defazio}
\author[1]{Konstantin Mishchenko}
\author[1]{Parameswaran Raman}
\author[1]{Hao-Jun Michael Shi}
\author[1]{Lin Xiao}

\affiliation[1]{Meta Superintelligence Labs, Menlo Park, CA 94025, USA \\
\texttt{\{adefazio, konstmish, params, hjmshi, linx\}@meta.com} \\
}

\abstract{
    We propose Generalized Primal Averaging (GPA), an extension of Nesterov's method that unifies and generalizes recent averaging-based optimizers like single-worker DiLoCo and Schedule-Free, within a non-distributed setting. While DiLoCo relies on a memory-intensive two-loop structure to periodically aggregate pseudo-gradients using Nesterov momentum, GPA eliminates this complexity by decoupling Nesterov's interpolation constants to enable smooth iterate averaging at every step. Structurally, GPA resembles Schedule-Free but replaces uniform averaging with exponential moving averaging. Empirically, GPA consistently outperforms single-worker DiLoCo and AdamW with reduced memory overhead. GPA achieves speedups of 8.71\%, 10.13\%, and 9.58\% over the AdamW baseline in terms of steps to reach target validation loss for Llama-160M, 1B, and 8B models, respectively. Similarly, on the ImageNet ViT workload, GPA achieves speedups of 7\% and 25.5\% in the small and large batch settings respectively. Furthermore, we prove that for any base optimizer with $\mathcal{O}(\sqrt{T})$ regret, where $T$ is the number of iterations, GPA matches or exceeds the original convergence guarantees depending on the interpolation constants.
}

\date{\today}
\correspondence{Aaron Defazio: \email{adefazio@meta.com}, Parameswaran Raman: \email{params@meta.com}}

\metadata[Code]{\url{https://github.com/facebookresearch/optimizers/tree/main/gpa}}

\begin{document}

\maketitle

\section{Introduction}
\label{section:intro}

As large language models (LLMs) demonstrate increasingly remarkable capabilities at scale \citep{achiam2023gpt, grattafiori2024llama3herdmodels, liu2024deepseek}, the pre-training phase has become one of the most resource intensive stages in the language model training pipeline. 
This has encouraged the development of training algorithms and optimizers that enhance the efficiency, scalability, and robustness of language model pre-training. 
One significant area of research is the design of training algorithms for scalable distributed learning.
In this area, the DiLoCo algorithm has emerged as the leading practical approach \citep{diloco,liu2024asynchronous,douillard2025streaming,charles2025communication}. 

Despite its practical success, the underlying reasons for DiLoCo's effectiveness remain poorly understood. Importantly, DiLoCo is not limited to distributed training: single-worker DiLoCo outperforms AdamW \emph{even in the non-distributed setting}. \citet{kallusky2025snoo} suggest that this is due to its novel combination of the Nesterov optimizer with the Lookahead method \citep{lookahead}, called Step-$K$ Nesterov. The method accumulates multiple updates from a base optimizer on an inner set of weights, forming what is called a \emph{pseudo-gradient}. It then applies Nesterov momentum to the pseudo-gradient to update an outer set of weights, and subsequently resets the inner weights to match the new values of the outer weights. On a 160 million parameter Llama model, single-worker DiLoCo achieves speedups of up to 6.32\% in terms of steps to reach the final validation loss by AdamW; see Figure~\ref{fig:bar_consolidated_valloss_vs_comm_intervals}.

\begin{figure}[htbp]
    \begin{subfigure}{0.54\textwidth}
        \includegraphics[width=\linewidth]{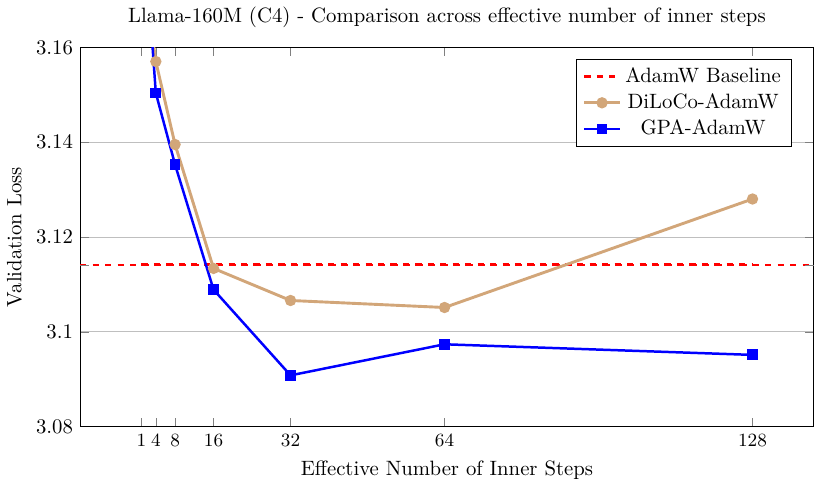}
        \caption{Both GPA and single-worker DiLoCo, when using AdamW as their base optimizer, outperform the tuned AdamW baseline for training a 160M parameter Llama model. Notably, increasing the number of inner steps (up to 64) improves the performance of single-worker DiLoCo. GPA instead updates the parameters at every step using a heuristic to choose interpolation constants that approximately match the number of inner steps for single-worker DiLoCo.}
        \label{fig:consolidated_valloss_vs_comm_intervals}
    \end{subfigure}
    \hfill
    \begin{subfigure}{0.44\textwidth}
        \begin{center}
            \includegraphics[width=\textwidth]{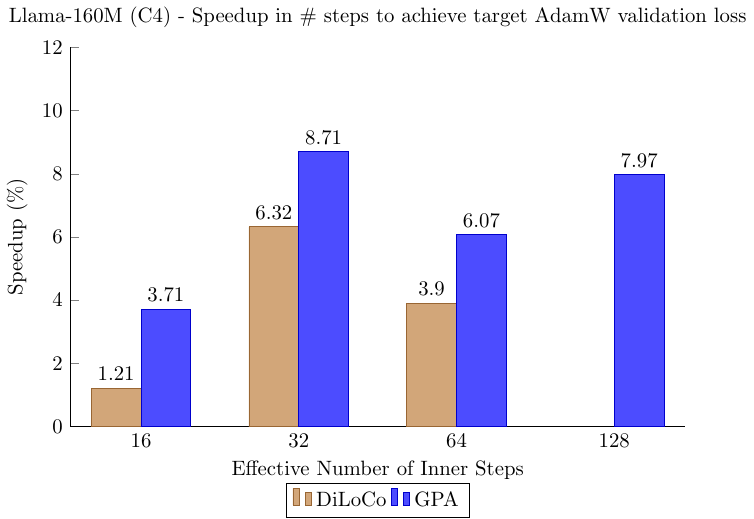}
        \end{center}
        \caption{Speedup achieved by single-worker DiLoCo and GPA measured in terms of reduction in number of steps required to attain the final validation loss obtained by AdamW, across different effective numbers of inner steps. GPA attains the highest speedup of 8.71\% when the effective inner steps is equal to 32. Single-worker DiLoCo does not exhibit any speedup for 128 inner steps.}
        \label{fig:bar_consolidated_valloss_vs_comm_intervals}
    \end{subfigure}
    \caption{Comparison of validation loss and speedup for AdamW, single-worker DiLoCo, and GPA on Llama-160M model. See Appendix~\ref{app:expmt_details_consolidated_speedup_1B} for similar results on Llama-1B.}
    \label{fig:consolidated_plots}
\end{figure}

Intriguingly, DiLoCo's performance initially improves as the number of inner steps increases. With each inner step, DiLoCo's outer weights drift farther from its inner weights, similar to meta-learning optimizers such as Reptile \citep{nichol2018reptile} and First-Order MAML \citep{finn2017model}. As a result, updates to the outer weights occur only at periodic intervals, causing information from the data to be integrated in a discontinuous, choppy manner rather than smoothly at every iteration. This restriction on information flow to the outer weights appears unnecessary from an optimization perspective, yet counterintuitively improves its performance; see Figure~\ref{fig:consolidated_valloss_vs_comm_intervals}.

Concurrently, the Schedule-Free optimizer recently won the AlgoPerf Algorithmic Efficiency challenge self-tuning track \citep{Dahl2023AlgoPerf,schedule-free}. Its core novelty lies in computing gradients at a point that interpolates between the uniform average of past weights and the current weights. 
Empirically, Schedule-Free matches the performance obtained by using learning rate schedules without using any schedule explicitly, while providing stronger theoretical last-iterate convergence guarantees similar to Polyak-Ruppert averaging \citep{ruppert,polyak,polyak1992acceleration}. However, its reliance on uniform averaging limits its flexibility and performance in some settings.

In this paper, we argue that these two lines of work -- DiLoCo and Schedule-Free -- are closely related and can be generalized and improved through a unified framework of \emph{primal averaging}. Specifically, our contributions are as follows:
\begin{itemize}
    \item We propose a novel generalization of Nesterov's method in its primal averaging formulation called \emph{Generalized Primal Averaging} (GPA). The method can be interpreted as a \emph{smoothed version of single-worker DiLoCo} that incrementally averages iterates at every step. It can also be viewed as a subtle change of Schedule-Free that replaces uniform averaging with exponential moving averaging through a decoupled interpolation parameter to improve its practical performance.
    \item In contrast to single-worker DiLoCo, GPA \emph{eliminates the two-loop structure}, thereby requiring only a single additional buffer with one less hyperparameter to tune. Because it incrementally averages iterates at every step, the method consistently exhibits more \emph{stable training behavior} than single-worker DiLoCo.
    \item Our experiments demonstrate that GPA \emph{consistently outperforms} single-worker DiLoCo and AdamW on dense 160 million and 1 billion parameter language models. 
    We validate our results on different modalities through an 8 billion parameter Llama code generation model and through a vision workload using ImageNet ViT in both small- and large-batch settings.
    In particular, on the Llama-160M, 1B, and 8B models, we find that GPA \emph{provides speedups} of 8.71\%, 10.13\%, and 9.58\%, respectively, in terms of steps to reach the baseline validation loss. 
    Likewise, GPA obtains speedups of 7\% and 25.5\% (small and large batch) on the ImageNet ViT workloads.
    \item We provide theoretical justification for GPA by proving convergence guarantees that show improved convergence over the base optimizer on stochastic convex and non-smooth functions under some settings.
\end{itemize}

\section{Background}
\label{section:background}

We frame language model pre-training as the expected risk minimization problem
\begin{equation}
\label{eq:expected_risk}
\min_{x \in \R^n} F(x) = \mathbb{E}_{\xi \sim \mathcal{D}}\left[f(x; \xi)\right],
\end{equation}
where $\xi \sim \mathcal{D}$ is drawn from an underlying stationary data distribution $\mathcal{D}$. We assume that each optimizer step has access to the stochastic minibatch gradient $g(x^{(t)}; \xi^{(t)}) \in \partial f(x^{(t)}; \xi^{(t)})$ evaluated at each iteration $t$ on a minibatch of data $\xi^{(t)}$, over a total of $T$ steps.\footnote{We assume that $f$ is convex for the convergence analysis, but we verify its performance on non-convex, possibly non-smooth functions.}

We also assume that the base optimizer is of the form $x^{(t + 1)} = x^{(t)} + \gamma^{(t)} d^{(t)}$ with learning rate $\gamma^{(t)} > 0$ and search direction $d^{(t)} \in \R^n$. The search direction is most commonly defined as $d^{(t)} = - H^{(t)} m^{(t)}$, where $m^{(t)} \in \R^n$ is a gradient estimator, and $H^{(t)} \in \R^{n \times n}$ is a symmetric positive definite preconditioner matrix. This includes popular methods such as SGD, Adam, Shampoo, SOAP, AdEMAMix, or Muon for different choices of $m^{(t)}$ and $H^{(t)}$ \citep{robbins1951stochastic,kingma2014adam,pmlr-v80-gupta18a,loshchilov2018decoupled,anil2020scalable,shi2023distributed,vyas2024soap,jordan2024muon,pagliardini2025the,eschenhagen2025purifying}.

\subsection{Different Formulations of Nesterov Momentum}

Nesterov momentum has played a critical role in optimization for deep learning \citep{pmlr-v28-sutskever13}. Despite its importance, there is still substantial confusion in the literature regarding Nesterov's formulation, as it can be written in at least seven different ways \citep{adefazio-curvedgeom2019}. These formulations are equivalent in the sense that a direct mapping exists between them, but they may not return the same iterate. 

For instance, Nesterov's method was popularized for deep learning in \emph{Sutskever's formulation} \citep{pmlr-v28-sutskever13}, which presents the algorithm as:
\begin{align}
\label{eq:sutskever_form}
\begin{split}
b^{(t)} & = \mu b^{(t - 1)} - \gamma^{(t)} g(x^{(t)} + \mu b^{(t-1)}; \xi^{(t)}), \\
x^{(t+1)} & =x^{(t)} + b^{(t)},
\end{split}
\end{align}
where $\mu > 0$ is the momentum hyperparameter and $b^{(t)} \in \R^n$ is the momentum buffer initialized at $b^{(0)} = 0$. An alternative formulation, which we call the \emph{modern formulation}, is used by software libraries such as \texttt{PyTorch}\footnote{\href{https://docs.pytorch.org/docs/2.8/generated/torch.optim.SGD.html}{PyTorch SGD documentation}} and \texttt{JAX}\footnote{\href{https://optax.readthedocs.io/en/latest/api/optimizers.html\#optax.sgd}{Optax (JAX) SGD documentation}} due to its ease of use:
\begin{align}
\label{eq:modern_form}
\begin{split}
b^{(t)} & = \mu b^{(t - 1)} + g(x^{(t)}; \xi^{(t)}),\\
x^{(t+1)} & = x^{(t)} - \gamma^{(t)} [\mu b^{(t)} + g(x^{(t)}; \xi^{(t)})].
\end{split}
\end{align}
In both formulations, we maintain a momentum buffer that averages the gradients seen throughout the training process. However, unlike Sutskever's formulation (\Cref{eq:sutskever_form}), the modern formulation (\Cref{eq:modern_form}) uses the iterate $x^{(t)}$ directly for the gradient computation, rather than the ancillary point $x^{(t)} + \mu b^{(t-1)}$, simplifying its practical implementation. If both formulations are run side-by-side with the same seed, they will evaluate gradients at exactly the same points, but their validation losses at iterates $x^{(t)}$ for each method will differ. 

Our approach instead builds upon a third form, which we call the \emph{primal averaging formulation}:
\begin{align}
\label{eq:primal_averaging_form}
\begin{split}
y^{(t)} & = \mu x^{(t)} + (1 - \mu)z^{(t)}, \\
z^{(t+1)} & = z^{(t)} - \gamma^{(t)} g(y^{(t)}; \xi^{(t)}),\\
x^{(t+1)} & = \mu x^{(t)} + \left(1-\mu \right)z^{(t+1)},
\end{split}
\end{align}
with $\mu \in [0, 1)$. The first mention of this three-sequence form that we are aware of is by \citet{lanaccel}, although it was only studied under a time-varying $\mu$. 

Unlike the Sutskever and modern formulations framed in equations \ref{eq:sutskever_form} and \ref{eq:modern_form}, the primal averaging formulation in \Cref{eq:primal_averaging_form} explicitly names two iterate sequences: a sequence where the gradients (or, more generally, the search directions) are computed at, i.e., the \emph{gradient computation sequence} $\{ y^{(t)} \}_{t = 1}^T$, as well as another sequence used for model evaluation that accumulates a running average of updated iterates $\{z^{(t)} \}_{t = 1}^T$, i.e., the \emph{model evaluation sequence} $\{ x^{(t)} \}_{t = 1}^T$. Since $y^{(t)}$ interpolates the smoothed sequence $x^{(t)}$ and unsmoothed sequence $z^{(t)}$, it increases the contribution of the gradient update to $y^{(t)}$ compared to $x^{(t)}$. This explicit formulation is convenient for implementation and theoretical analysis, and naturally leads to a view of acceleration as built upon \emph{iterate averaging}, rather than from the physics-inspired intuition of \emph{gradient averaging} behind momentum that is more commonly introduced.

We summarize the relationship between the modern and primal averaging formulations in Proposition \ref{prop:primal-modern} below.
\begin{prop}
    \label{prop:primal-modern}
    Given fixed learning rates $\gamma_{\primal}, \gamma_{\modern} > 0$, the primal averaging formulation of Nesterov's method (\Cref{eq:primal_averaging_form}) is equivalent to its modern formulation (\Cref{eq:modern_form}) in the sense that 
    \begin{equation}
        \begin{aligned}
            y_{\primal}^{(t)} &= x_{\modern}^{(t)} ~~~ \text{and} \\ ~~~ b_{\modern}^{(t)} &= \frac{1}{\left(1 - \mu \right)\gamma_{\primal}}\left(x_{\primal}^{(t)} - x_{\primal}^{(t + 1)}\right),    
        \end{aligned}
    \end{equation}
    when $\mu_{\primal} = \mu_{\modern} = \mu$ and $\left(1 - \mu \right) \gamma_{\primal} = \gamma_{\modern}$.
\end{prop}

The proof of this simple statement is rather technical, so we defer it to Appendix~\ref{app:proofs}. Similar formulations and equivalences can be derived for Polyak momentum \citep{polyak1964some, defazio2020mom, ziyin2020laprop}; see Appendix~\ref{app:polyak}. We summarize these formulations and their differences in \Cref{tab:nesterov_formulations}.

\textbf{Remark.} It is important to acknowledge that the equivalence between the primal averaging and modern formulations of Nesterov momentum holds only when the learning rates are \emph{constant}. When learning rate schedules are introduced, achieving this equivalence would require the momentum parameter to vary with each iteration. Furthermore, the restriction on the choice of $\mu$ differs between the modern and primal averaging formulations. These different interpretations based on \emph{gradient averaging} versus \emph{iterate averaging} produce different perspectives for hyperparameter tuning, which can have a significant impact on the algorithm's practical performance.

\begin{table*}[ht!] 
\centering
\caption{Summary of Nesterov, Schedule-Free, and GPA formulations. Here, $\mu$ is the momentum parameter, $\gamma^{(t)}$ is the learning rate, and $g(\cdot; \xi^{(t)})$ denotes the stochastic gradient. For a summary of Polyak momentum formulations, please refer to Appendix~\ref{app:polyak}.}
\label{tab:nesterov_formulations}
\begin{tabular}{@{}p{3.5cm} p{6.5cm} p{5cm}@{}}
\toprule
\textbf{Formulation} & \textbf{Update Equations} & \textbf{Notes} \\
\midrule
\rowcolor{colorlightgray} \makecell[l]{\textbf{Sutskever (Classical)} \\ \citep{pmlr-v28-sutskever13}} &
$\begin{aligned}
b^{(t)} &= \mu b^{(t-1)} - \gamma^{(t)} g(x^{(t)} + \mu b^{(t-1)}; \xi^{(t)}) \\
x^{(t+1)} &= x^{(t)} + b^{(t)}
\end{aligned}$
& \makecell[l]{Gradient is evaluated at the \\ lookahead point.} \\
\midrule
\rowcolor{colorlightgray} \textbf{Modern (PyTorch/JAX)} &
$\begin{aligned}
b^{(t)} &= \mu b^{(t-1)} + g(x^{(t)}; \xi^{(t)}) \\
x^{(t+1)} &= x^{(t)} - \gamma^{(t)} [\mu b^{(t)} + g(x^{(t)}; \xi^{(t)})]
\end{aligned}$
& \makecell[l]{Gradient is evaluated at the \\ current point.} \\
\midrule
\rowcolor{colorlightgray} \makecell[l]{\textbf{Primal Averaging Variant} \\ \textbf{of Nesterov} \\ \citep{lanaccel}} &
$\begin{aligned}
y^{(t)} &= \mu x^{(t)} + (1-\mu) z^{(t)} \\
z^{(t+1)} &= z^{(t)} - \gamma^{(t)} g(y^{(t)}; \xi^{(t)}) \\
x^{(t+1)} &= \mu x^{(t)} + (1-\mu) z^{(t+1)}
\end{aligned}$
& \makecell[l]{Explicit separation of gradient \\ and model evaluation sequences.} \\
\midrule
\rowcolor{colorlightgray} \makecell[l]{\textbf{Schedule-Free} \\ \citep{schedule-free}} &
$\begin{aligned}
y^{(t)} &= \mu x^{(t)} + (1-\mu) z^{(t)} \\
z^{(t+1)} &= z^{(t)} - \gamma g(y^{(t)}; \xi^{(t)}) \\
x^{(t+1)} &= \frac{t}{t+1} x^{(t)} + \left(1-\frac{t}{t+1}\right) z^{(t+1)}
\end{aligned}$
& \makecell[l]{Uniform averaging; learning rate \\ schedule-free.} \\
\midrule
\rowcolor{colorblue} \textbf{GPA (Ours)} &
$\begin{aligned}
y^{(t)} &= \mu_y x^{(t)} + (1-\mu_y) z^{(t)} \\
z^{(t+1)} &= z^{(t)} - \gamma^{(t)} g(y^{(t)}; \xi^{(t)}) \\
x^{(t+1)} &= \mu_x x^{(t)} + (1-\mu_x) z^{(t+1)}
\end{aligned}$
& \makecell[l]{Decoupled interpolation of gradient \\ and model evaluation sequences; \\ requires learning rate schedule.} \\
\bottomrule
\end{tabular}
\end{table*}

\subsection{Single-Worker DiLoCo and its Weaknesses}

DiLoCo was originally introduced as a distributed algorithm for cross-datacenter training \citep{diloco}. The method computes multiple inner steps of a base optimizer on the \emph{inner weights}, then applies Nesterov (\Cref{eq:modern_form}) on the average \emph{pseudo-gradient}, the difference between the previous and updated inner model weights, to update the \emph{outer weights}. The inner weights are then reset to outer weights. 

DiLoCo requires storing two additional optimizer states of the same shape as the model parameters: the momentum buffer $b^{(t)}$ and the current model parameters $x^{(t)}$ (also known as the \emph{outer weights}). DiLoCo's handling of \emph{fast} inner weights and \emph{slow} outer weights can be interpreted as a modified Lookahead method that applies Nesterov momentum to the outer weight updates \citep{lookahead}. The method was recently analyzed in \citet{khaled2025understanding}, and demonstrated significant compute factor gains in the non-distributed setting in \citet{kallusky2025snoo}. A simplified version of \emph{non-distributed} or \emph{single-worker DiLoCo} with $H$ inner steps of the base optimizer can be described as:
\begin{align}
p^{(t)} & = x^{(t)} - \baseoptiter(x^{(t)}; \{\gamma^{(j)}\}_{j = 1}^H, H) \nonumber\\
b^{(t)} & = \mu b^{(t - 1)} + p^{(t)} \label{eq:diloco}
\\
x^{(t+1)} & = x^{(t)} - \tilde{\gamma} [\mu b^{(t)} + p^{(t)}], \nonumber
\end{align}
where $\tilde{\gamma} > 0$ is the outer learning rate and $\baseoptiter$ applies $H$ inner steps of the base optimizer to the iterate $x^{(t)}$ with inner learning rates $\{\gamma^{(j)}\}_{j=1}^H$. While DiLoCo originally introduced AdamW as the base optimizer, DiLoCo has been generalized to other optimizers such as Muon \citep{therien2025muloco}. A complete description of the algorithm is provided in Appendix~\ref{app:algorithm-details}. As noted in \citet{kallusky2025snoo}, applying Nesterov on the pseudo-gradient with multiple inner steps is capable of surpassing the performance of the base optimizer alone, which explains DiLoCo's ability to match the synchronous baseline, such as AdamW, in the multi-worker setting. 

\textbf{Weaknesses in DiLoCo's hierarchical framework.} 

However, this two-level structure is undesirable. 
From an \emph{algorithmic perspective}, one would prefer to average iterates on-the-fly, as opposed to averaging trajectories that implicitly contain multiple iterations of the base optimizer. 
From the \emph{users' perspective}, the two-level structure introduces an additional copy of the model weights required to compute the pseudo-gradient, and introduces additional hyperparameters to tune, e.g., the inner and outer learning rates, momentum, and number of inner steps. Lastly, from the \emph{distributed training perspective}, DiLoCo couples the number of inner steps as a hyperparameter for both local SGD as well as for the modified Nesterov algorithm, causing the algorithm's performance to  counterintuitively improve as the number of inner steps increases. One would instead expect that communicating more often should always be beneficial. These challenges motivate the development of a new algorithm that removes the two-level structure while offering a separate hyperparameter that can smoothly average the observed iterates at every iteration.

\subsection{Schedule-Free Learning}
In parallel, Schedule-Free learning (SF) \citep{schedule-free} was recently proposed as a wrapper to any base optimizer using a variant of the primal averaging formulation of Nesterov's method (\Cref{eq:primal_averaging_form}) for hyperparameter-free learning:
\begin{align}
\label{eq:sf}
\begin{split}
y^{(t)} & = \mu x^{(t)} + (1 - \mu) z^{(t)}\\
z^{(t+1)} & =z^{(t)} - \gamma g(y^{(t)}; \xi^{(t)})\\
x^{(t+1)} & = \frac{t}{t+1} x^{(t)} + \Bigl(1 - \frac{t}{t+1}\Bigr)z^{(t+1)}.
\end{split}
\end{align}

Originally designed to eliminate the need for manually specified learning rate schedules, Schedule-Free has demonstrated the surprising ability to not only match, but even surpass the practical performance of the original base optimizer. This is done by \emph{decoupling} the momentum hyperparameter used in the $x^{(t)}$ and $y^{(t)}$ sequences, unlike the standard primal averaging formulation of Nesterov (\Cref{eq:primal_averaging_form}). Through the choice of $\mu$, the method interpolates between uniform Polyak-Ruppert averaging and stochastic primal averaging \citep{ruppert,polyak,tao2018}. 

Ignoring the hyperparameter-free learning problem, one could alternatively replace uniform averaging with exponential moving averaging of the iterates, which is commonly used in practice \citep{morales2024exponential}. This alternative suggests a different generalization of Nesterov momentum that may offer the potential flexibility necessary to reproduce DiLoCo's convergence gains without the two-level structure.

\section{Generalized Primal Averaging (GPA)}
\label{section:gpa}

By decoupling the constants for the model evaluation and gradient computation sequences in the primal averaging formulation of Nesterov's method (\Cref{eq:primal_averaging_form}) and leveraging the observation of using exponential moving averaging in lieu of uniform averaging in Schedule-Free (\Cref{eq:sf}), we introduce the \emph{Generalized Primal Averaging} (GPA) framework:

\begin{align}
\label{eq:gpa}
\begin{split}
y^{(t)} & = \mu_y x^{(t)} + (1 - \mu_y) z^{(t)} \\
z^{(t+1)} & =z^{(t)}-\gamma^{(t)} g(y^{(t)}; \xi^{(t)})\\
x^{(t+1)} & = \mu_x x^{(t)} + \left(1 - \mu_x \right) z^{(t+1)}.
\end{split}
\end{align}

Here, $\mu_x \in [0, 1)$ and $\mu_y \in [0, 1]$ are independent hyperparameters that separately control the degree of interpolation used to maintain the model evaluation sequence $x^{(t)}$ and gradient computation sequence $y^{(t)}$. The additional hyperparameter $\mu_x$ serves as a smoothing or exponential moving average parameter that replaces Polyak-Ruppert averaging in Schedule-Free, while $\mu_y$ controls the amount of information flow into $y^{(t)}$. The complete pseudocode for a general base optimizer is provided in Algorithm~\ref{alg:gpa}. 

Unlike the modern formulation of Nesterov momentum (\Cref{eq:modern_form}) or DiLoCo (\Cref{eq:diloco}) built on (pseudo-)gradient averaging, GPA is defined based on the \emph{primal} or \emph{iterate averaging framework}. We argue that this provides a more meaningful characterization of the method. For example, the primal averaging interpretation naturally extends to other search directions by replacing $-g(y^{(t)}; \xi^{(t)})$ with the search direction $d^{(t)}$ evaluated at $y^{(t)}$. This extension is not intuitive from the gradient averaging perspective, as it would translate to averaging search directions (with potentially different, evolving preconditioners) in the momentum buffer.

\textbf{Learning rate schedules.} By replacing Polyak-Ruppert averaging with exponential moving averaging, GPA is not inherently schedule-free and requires the use of a learning rate schedule. To see why, observe that Polyak averaging places increasingly less weight $1 / (t + 1)$ on the most recent iterate $z^{(t + 1)}$, which plays a similar role to learning rate scheduling \citep{sandler2023trainingtrajectories,schedule-free}. GPA instead places a constant weight $\mu_x$ on the most recent iterate $z^{(t + 1)}$ by leveraging an exponential moving average, thereby requiring a learning rate schedule compared to Schedule-Free. This is reflected theoretically in their average-iterate versus last-iterate convergence properties.

\textbf{Degenerate cases.} The choice of $\mu_x$ and $\mu_y$ enables GPA to recover different averaging methods:
\begin{itemize}[topsep=0pt, partopsep=0pt, parsep=0pt]
    \item When $\mu_y = 1$, $x^{(t)} = y^{(t)}$ and we recover stochastic primal averaging, or equivalently, LaProp \citep{defazio2020mom, ziyin2020laprop}; see Appendix~\ref{app:algorithm-details}.
    \item When $\mu_y = 0$, $x^{(t)}$ and $z^{(t)} = y^{(t)}$ become decoupled and we recover exponential moving averaging of the iterates \citep{morales2024exponential}.
    \item When $\mu_x = 0$, $x^{(t)} = y^{(t)} = z^{(t)}$ for any choice of $\mu_y$, and GPA reverts to the base optimizer. 
\end{itemize}

\textbf{Memory-efficient implementation.} When $\mu_y > 0$, GPA can be implemented with only one extra copy of the model weights -- specifically, by storing $y^{(t)}$ and reconstructing $x^{(t)}$ from $y^{(t)}$ and $z^{(t)}$ during evaluation -- unlike DiLoCo, which demands more memory overhead due to the global parameter and momentum buffers. The pseudocode is provided in Appendix~\ref{app:algorithm-details}.

\textbf{Other properties.} Because $\mu_x, \mu_y \in [0, 1]$, GPA preserves modular norm bounds of the model parameters \citep{bernstein2024modular}. More details on this property is provided in Appendix~\ref{app:algorithm-details}. GPA can also be interpreted as a special case of the Triple Momentum framework \citep{7967721}, which utilizes three momentum terms to achieve tighter global linear convergence bounds for deterministic, smooth, strongly convex functions compared to Nesterov's method.

\begin{algorithm}[t]
\begin{algorithmic}[1]
    \Require Initial iterate $x^{(1)}$, learning rate schedule $\gamma^{(t)} > 0$, weight decay $\lambda \geq 0$, interpolation parameters $\mu_x, \mu_y \in [0, 1)$, base optimizer $\baseopt$.
    \State $z^{(1)} = x^{(1)}$
    \For{$t = 1, ..., T$}
        \State $y^{(t)} = \mu_y x^{(t)} + (1 - \mu_y)z^{(t)}$ \Comment{Update gradient computation point $y^{(t)}$.}
        \State $g^{(t)} \in \partial f(y^{(t)}; \xi^{(t)})$ \Comment{Gradient is evaluated at $y^{(t)}$}.
        \State $d^{(t)} = \baseopt(g^{(t)})$ \Comment{Compute base optimizer's search direction.}
        \State $z^{(t+1)} = (1 - \gamma^{(t)} \lambda) z^{(t)} + \gamma^{(t)} d^{(t)}$ \Comment{Update $z^{(t)}$ iterate.}
        \State $x^{(t+1)} = \mu_x x^{(t)} + \left(1 - \mu_x \right)z^{(t+1)}$ \Comment{Update weighted iterate average $x^{(t)}$.}
    \EndFor\\
    \Return $x^{(T)}$
\end{algorithmic}
\caption{\label{alg:gpa}Generalized Primal Averaging (GPA)}
\end{algorithm}

\subsection{Interpreting GPA as Smoothed DiLoCo}
\label{sec:gpa-diloco}

As seen in Figure~\ref{fig:consolidated_valloss_vs_comm_intervals}, increasing the number of inner steps leads to improved performance for single-worker DiLoCo. However, the underlying reasons for this behavior are not understood. By examining DiLoCo from the lens of GPA in \Cref{eq:gpa} and comparing it with the more restrictive Nesterov formulation in \Cref{eq:primal_averaging_form}, we can develop a deeper intuition for DiLoCo's inner workings. 

Suppose that we increase the number of inner steps in DiLoCo and want to maintain the same level of smoothing on the average iterate $x^{(t)}$. One may attempt to increase $\mu$ in Nesterov (\Cref{eq:primal_averaging_form}) to decrease the weight on the current iterate $z^{(t+1)}$. However, since $\mu$ controls both the amount of smoothing in $x^{(t)}$ \emph{and} the amount of interpolation used to update $y^{(t)}$, strictly increasing $\mu$ would \emph{decrease the recency of information from $z^{(t)}$ in $y^{(t)}$} by a factor of $\mu^2$, resulting in significantly different algorithmic behavior. Numerically, we validate that tuning $\mu$ alone in Nesterov's primal averaging formulation is not sufficient to reach the performance of DiLoCo; see Appendix~\ref{app:expmt_details}.

GPA addresses this limitation by \emph{decoupling the two roles of $\mu$ into separate hyperparameters}: $\mu_x$ for the model evaluation sequence and $\mu_y$ for the gradient computation sequence. By controlling these two interpolation constants independently, we can smooth $x^{(t)}$ similarly without changing the amount of information introduced into $y^{(t)}$. This smoothing is depicted in Figure~\ref{fig:visual} on a simple deterministic quadratic problem. For a small number of inner steps, the methods closely align, but for a larger number of inner steps, their behavior diverges.  

\begin{figure}[htbp]
\centering \includegraphics[width=0.45\linewidth]{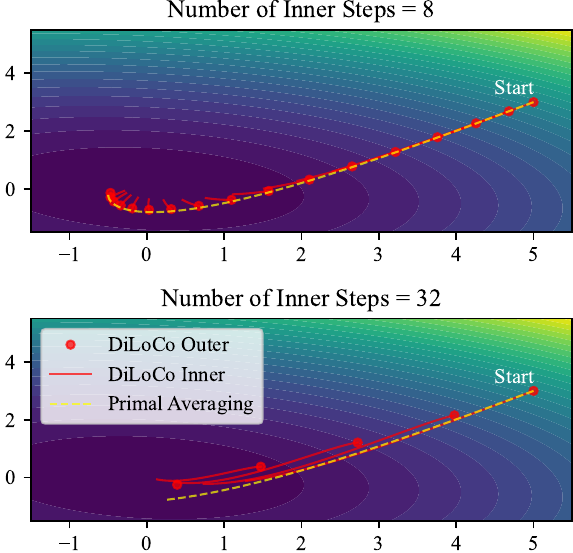}
  \caption{Comparison of DiLoCo and GPA's trajectories on a deterministic quadratic problem. The outer iterates of DiLoCo are shown as red points, and the inner iterates as thin red lines.}
  \label{fig:visual}
\end{figure}

\textbf{Tuning GPA from DiLoCo.} This intuition provides practical guidelines for converting a hyperparameter recipe for DiLoCo to GPA. Given an optimal number of inner steps $H$ and momentum parameter $\mu$ in DiLoCo, we observe for GPA that $x^{(t+H)} = \mu_x^H x^{(t)} + (1-\mu_x)\sum_{k=0}^{H-1} \mu_x^k z^{(t+H-k)}$. Therefore, to match the coefficient in front of $x^{(t)}$ with DiLoCo, one can set $\mu_x = \mu^{1/H}$ while keeping $\mu_y \approx \mu$. With commonly used values $\mu = 0.9$ and $H = 32$, we obtain $\mu_x \approx 0.9967$ and $\mu_y \approx 0.9$. We leverage this heuristic to determine an effective number of inner steps used in Figure~\ref{fig:consolidated_plots}. See Table~\ref{tab:diloco_gpa_mu_x} for exact values mapping inner steps in DiLoCo to GPA coefficient $\mu_{x}$.

\textbf{Tradeoffs with DiLoCo.} GPA not only outperforms DiLoCo, but does so with \emph{fewer hyperparameters} and \emph{lower memory requirements}. While DiLoCo requires four hyperparameters, e.g., the inner and outer learning rate, momentum hyperparameter, and number of inner steps, the memory-efficient implementation of GPA reduces this to just three: the learning rate and two momentum parameters. This simplification is possible because DiLoCo's practical performance is governed by an effective learning rate that couples the effect of the inner and outer learning rates ($\gamma^{(t)}$ and $\tilde{\gamma}$). On the other hand, GPA requires more FLOPs per-iteration, while DiLoCo amortizes its additional compute cost across multiple inner steps (see Appendix~\ref{app:mem_tradeoffs_gpa} for additional details).

\section{Experiments}
\label{section:expts}

In this section, we assess the effectiveness of GPA on both language model pre-training and computer vision workloads. For language modeling, we compare against baselines AdamW, DiLoCo and Schedule-Free (SF), while for computer vision experiments, we compare GPA against AdamW. For DiLoCo, Schedule-Free and GPA methods, we use AdamW as the base optimizer (DiLoCo-AdamW, Schedule-Free and GPA-AdamW, respectively).

\begin{table*}[ht!] 
\centering
\caption{Final validation loss versus effective number of inner steps $H$ for different optimizers on {\bf Llama-160M} and {\bf Llama-1B} models. Highlighted in bold is the lowest validation loss obtained across all inner step configurations $H$.} 
\begin{tabular}{l|ccccc|cccc}
\toprule
\multicolumn{1}{c}{} & \multicolumn{5}{c}{\bf Llama-160M} & \multicolumn{4}{c}{\bf Llama-1B} \\
{\bf Method} &
{$H = 8$} & {$H = 16$} & {$H = 32$} & {$H = 64$} & {$H = 128$} &
{$H = 16$} & {$H = 32$} & {$H = 64$} & {$H = 128$} \\
\midrule
\rowcolor{colorlightgray} AdamW 
& 3.1141 & \textcolor{mygray}{3.1141} & \textcolor{mygray}{3.1141} & \textcolor{mygray}{3.1141} & \textcolor{mygray}{3.1141}
& 2.6749 & \textcolor{mygray}{2.6749} & \textcolor{mygray}{2.6749} & \textcolor{mygray}{2.6749} \\
\rowcolor{colorlightgray} SF-AdamW 
& 3.1089 & \textcolor{mygray}{3.1089} & \textcolor{mygray}{3.1089} & \textcolor{mygray}{3.1051} & \textcolor{mygray}{3.1089}
& \textbf{2.638} & \textcolor{mygray}{2.638} & \textcolor{mygray}{2.638} & \textcolor{mygray}{2.638} \\
\rowcolor{colorlightgray} DiLoCo-AdamW 
& \textcolor{mygray}{3.1395} & \textcolor{mygray}{3.1133} & \textcolor{mygray}{3.1066} & 3.1051 & \textcolor{mygray}{3.128}
& \textcolor{mygray}{2.6737} & \textcolor{mygray}{2.6577} & \textcolor{mygray}{2.6572} & 2.6558 
\\
\rowcolor{colorblue} GPA-AdamW 
& \textcolor{mygray}{3.1353} & \textcolor{mygray}{3.1089} & \textbf{3.0908} & \textcolor{mygray}{3.0974} & \textcolor{mygray}{3.0951}
& \textcolor{mygray}{2.6639} & \textcolor{mygray}{2.6602} & 2.645 & \textcolor{mygray}{2.6553} \\
\bottomrule
\end{tabular}
\label{tab:consolidated_valloss_vs_comm_intervals}
\end{table*}

\begin{figure*}[ht] 
    \begin{subfigure}{0.49\textwidth}
        \includegraphics[width=\linewidth]{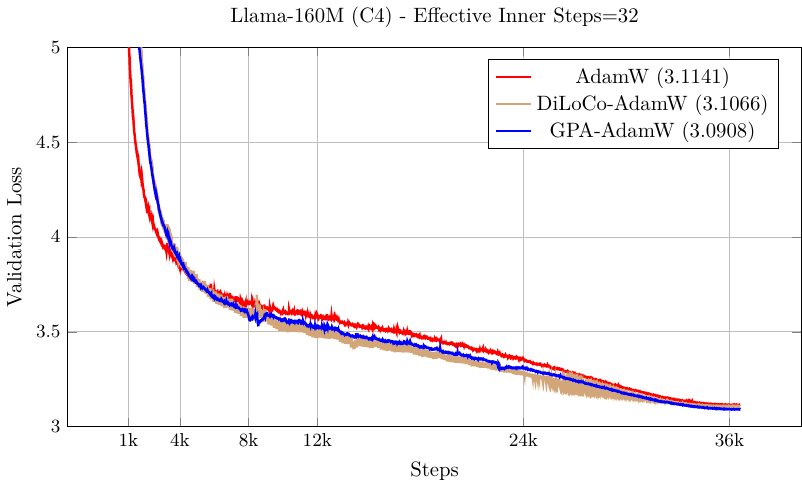}
        \caption{Comparison of AdamW, DiLoCo, and GPA with a fixed (effective) number of inner steps ($H = 32$).}
        \label{fig:comm_int32_valloss_vs_steps}
    \end{subfigure}
    \qquad
    \begin{subfigure}{0.49\textwidth}
        \includegraphics[width=\linewidth]{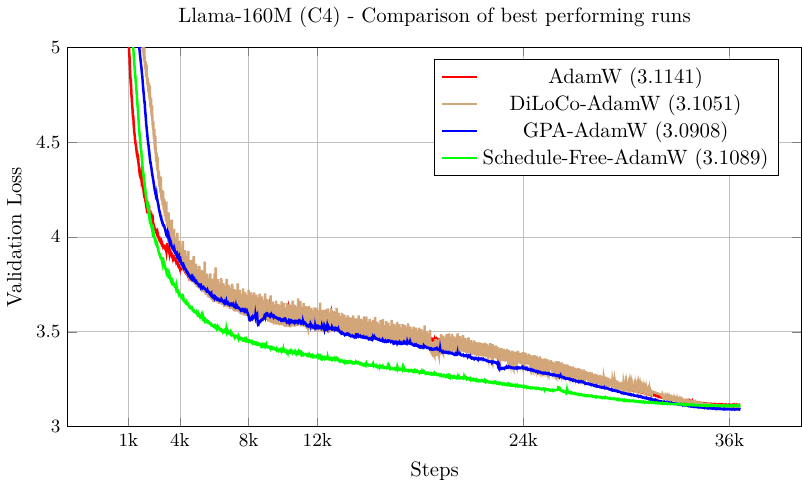}
        \caption{Comparison of AdamW, DiLoCo, GPA and Schedule-Free with optimal hyperparameters.}
        \label{fig:best_runs_valloss_vs_steps}
    \end{subfigure}
    \caption{Comparison of the validation loss against the number of steps for different optimizers on the Llama-160M workload.}
    \label{fig:comm_optimizers}
\end{figure*}

\subsection{Language Model Pre-Training}
We conduct experiments on multiple scales of  Llama models: {\bf 160 million} and {\bf 1 billion} parameters. These 160M and 1B models are pre-trained on the C4 dataset from scratch \citep{2019t5} using a token budget of roughly 9.6 billion and 50 billion tokens, respectively \citep{hoffmann2022training}. All of our small experiments are conducted on two nodes equipped with two GB200 GPUs (184 GB of memory) while the larger scale model experiments utilize four nodes (with a total of 8 GPUs). Comprehensive details on batch size, sequence length, and hyperparameter sweeps can be found in Appendix~\ref{app:expmt_details}. Note that both the Llama-160M and 1B experiments are performed in an overtrained setting.

\textbf{Performance across number of inner steps.} 
In Table \ref{tab:consolidated_valloss_vs_comm_intervals}, we provide the final validation loss for each method for different effective number of inner steps. 
Consistent with Figure~\ref{fig:consolidated_valloss_vs_comm_intervals}, GPA-AdamW outperforms both DiLoCo-AdamW and AdamW across all settings of the effective number of inner steps.

\textbf{Convergence behavior.}
Figure~\ref{fig:comm_int32_valloss_vs_steps} shows the validation loss curves on Llama-160M for AdamW, DiLoCo-AdamW, and GPA-AdamW for the case where the number of inner steps is 32. In this case, $\mu_{x}$ has been tuned to match the number of inner steps; see Table~\ref{tab:diloco_gpa_mu_x} in Appendix~\ref{app:expmt_details} for details. In Figure~\ref{fig:best_runs_valloss_vs_steps}, we compare GPA against the baselines (also including Schedule-Free-AdamW) by choosing the best performing runs over all hyperparameter choices including the effective number of inner steps. GPA-AdamW converges faster than both DiLoCo and AdamW throughout the entire training run. The training curves for GPA-AdamW are also noticeably smoother and more stable compared to the other methods. We observe that GPA-AdamW can handle higher learning rates compared to DiLoCo.

\subsection{Code Generation Model Pre-Training} 

We also evaluate GPA on an 8 billion parameter Llama model designed for code generation; see Appendix~\ref{app:expmt_details_hps_llama_8b} for more details. We find that GPA with any $\mu_y\in\{0.7, 0.8, 0.8\}$ outperforms AdamW throughout the course of training and achieves a better final validation loss (see Figure~\ref{fig:llama3_8b_code}, which plots GPA with $\mu_y=0.7$).

\begin{figure}[hbt!]
    \centering
    \includegraphics[width=0.6\linewidth]{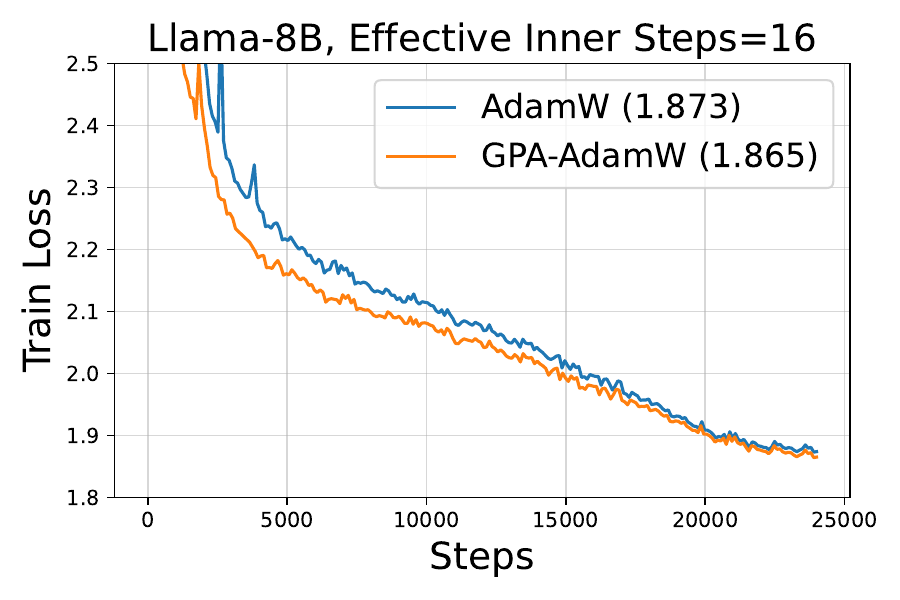}
    \caption{Comparison of AdamW and GPA on the Llama-8B code generation model using 100 billion tokens.}
    \label{fig:llama3_8b_code}
\end{figure}

\subsection{Vision Transformer Model Training}

To validate our method on a computer vision task, we train a ViT-S/16 model from \texttt{timm} on ImageNet with data augmentations from the repository. We train this under two batch size settings: (1) a small-batch setup with a batch size of 4,096 for 300 epochs; and (2) a large-batch setup with a batch size of 16,384 for 300 epochs. We tuned the methods separately in both settings, using the average over 2 random seeds to select the best hyperparameters, then ran the best-performing selection on 12 random seeds in total. For all methods, we used gradient clipping with norm 1 and a linear learning rate warmup over the first 5 epochs. We annealed the learning rate with cosine decay to $\times 0.001$ of the peak learning rate for all methods except Schedule-Free, which used constant learning rate after the warmup. Our results are that GPA outperforms AdamW by a clear margin in terms of validation accuracy throughout the course of training (see Figures~\ref{fig:vit_experiments}). For further details on our hyperparameter tuning, see Appendix~\ref{app:expmt_details}.

\begin{figure}[hbt!]
    \centering
    \includegraphics[width=0.48\linewidth]{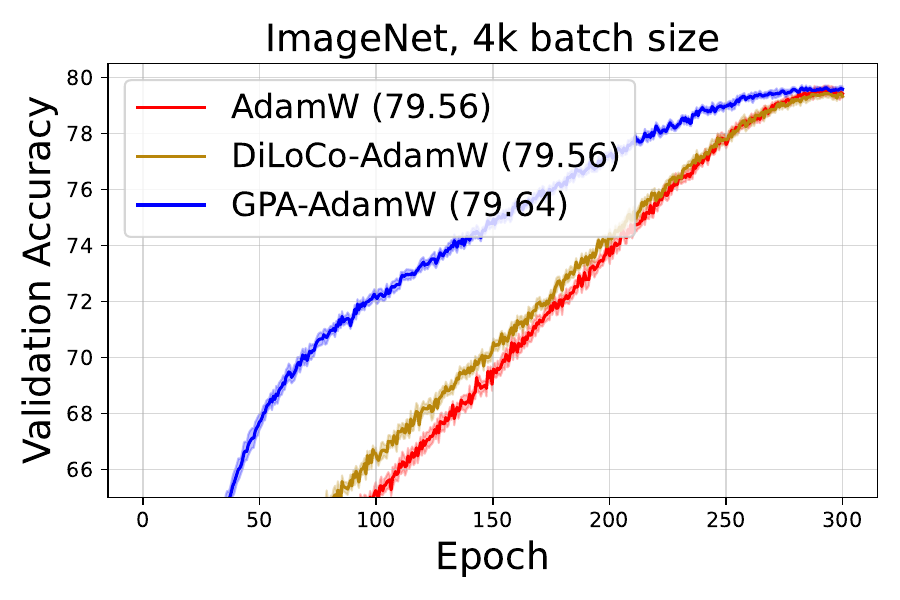}
    \includegraphics[width=0.48\linewidth]{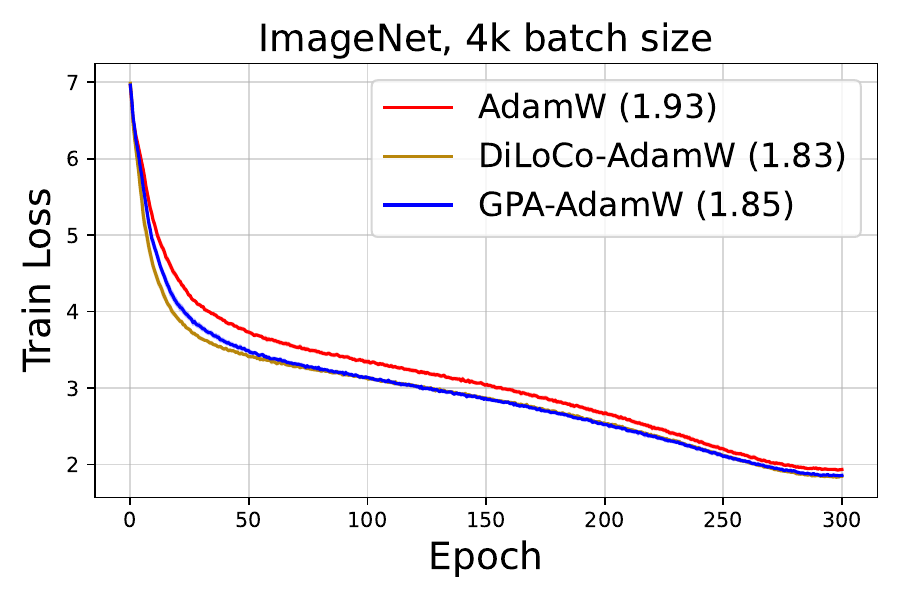}
    \caption{Comparison of AdamW, DiLoCo, and GPA on ImageNet ViT-S/16 from \texttt{timm} with data augmentations using a {\bf batch size of 4,096 samples}.}
    \label{fig:vit_experiments}
\end{figure}

\section{Convergence Theory}
\label{sec:theory}

Using the theoretical developments underpinning Schedule-Free learning, we can derive a convergence bound for GPA given any base optimizer that has a regret bound, using the framework of online-to-batch conversion \citep{cesa2004generalization}. We will use the Bregman divergence of $F$ defined as $B_F(a,b) = F(a)-F(b) - \langle \nabla F(b), a-b\rangle$ for $a, b \in \R^n$.

\begin{thm}
\label{thm:genbregman-gpa}
Let $F$ be a convex function and assume that there exists a minimizer $x_*$ that minimizes $F$. Let $\xi^{(1)},\dots,\xi^{(T)}$ be a sequence of i.i.d. random variables. Suppose that we are given arbitrary updates $z^{(1)},\dots,z^{(T)}$ from a base optimizer within the Generalized Primal Averaging framework (Equation~\ref{eq:gpa}). Then for $\mu_x, \mu_y \in [0, 1)$ and average iterate $\bar{x}^{(T)} = \frac{1}{T} \sum_{t=1}^{T} x^{(t)}$, we have the bound
\begin{align*}
\mathbb{E}[F(\bar{x}^{(T)})-F(x_{*})] & \leq \frac{1}{T}\sum_{t=1}^{T}\mathbb{E}[\langle\nabla F(y^{(t)}),z^{(t)}-x_{*}\rangle]  + \frac{\mu_x}{1-\mu_x}\frac{1}{T}\mathbb{E}\left[F(x^{(1)})-F(x_*)\right]\\
 & \qquad -\frac{1}{1-\mu_y}\frac{1}{T}\sum_{t=1}^{T} \mathbb{E}[B_{F}(y^{(t)},x^{(t)})] - \frac{\mu_y}{1-\mu_y}\frac{1}{T}\sum_{t=1}^{T}\mathbb{E}[B_{F}(x^{(t)},y^{(t)})] \\
 & \qquad -\frac{\mu_x}{1-\mu_x}\frac{1}{T}\sum_{t=1}^{T} \mathbb{E}[B_{F}(x^{(t-1)},x^{(t)})] .
\end{align*}
\end{thm}

\begin{corollary}
    Assume that the base optimizer has the regret guarantee $\sum_{t=1}^{T}\mathbb{E}[\langle\nabla F(y^{(t)}),z^{(t)}-x_{*}\rangle] = \mathcal{O}(\sqrt{T})$. Then:
    \begin{align*}
    \mathbb{E}[F(\bar{x}^{(T)})-F(x_{*})]
    = \mathcal{O} \left( \frac{1}{\sqrt{T}} \right).
\end{align*}
\end{corollary}

{\bf Remarks on Theorem \ref{thm:genbregman-gpa}}:
\begin{itemize}
    \item The first term on the right-hand side of the regret bound is the average regret of the base optimizer. This term captures the convergence rate from the base optimizer.
    \item The second term has a positive term, which decays at a rate of $1/T$, which is typically faster than the decay of the term in the first row.
    \item All remaining Bregman divergence terms are negative, and so are potentially beneficial. If $\mu_x$ and $\mu_y$ are chosen such that the negative terms dominate the positive second term, then GPA will converge faster than the base optimizer. 
    \item The same terms appear in the convergence guarantees for Schedule-Free methods, and can explain when they may work better. For strongly convex problems, such Bregman divergences were used to get $\mathcal{O}(1/T)$ convergence.
    \item Unlike the guarantees for Schedule-Free, our convergence bound is for the average iterate. For best performance, a learning rate schedule should be used and the last iterate returned \citep{defazio2023when}.
    \item Our bound indicates that GPA will be faster than the base optimizer when the objective function varies nonlinearly between consecutive iterates and between $x^{(t)}$ and $y^{(t)}$. 
\end{itemize}

\section{Conclusion}
GPA generalizes Nesterov momentum by decoupling interpolation constants for gradient computation and model evaluation. On both small and large-scale language models, GPA outperforms single-worker DiLoCo while eliminating its complex two-loop structure, simplifying hyperparameter tuning, and reducing memory overhead.

Future research should validate GPA across diverse architectures and modalities, exploring compatibility with optimizers like Shampoo, SOAP, and Muon, as well as hyperparameter transfer techniques like $\mu$P~\citep{yang2021tensor,yang2022tensor}. Our current convergence analysis is also limited to the convex setting; further research is required to fully characterize GPA’s advantages over base optimizers both in the convex and non-convex settings.

Finally, GPA’s decoupling of parameters also enables new avenues for distributed training. In DiLoCo, the number of inner steps serves as a coupled hyperparameter for both Step-$K$ Nesterov and local SGD, leading to the undesirable finding that increasing the number of inner steps can improve convergence -- contrary to standard local SGD intuition. By introducing a tunable, continuous smoothing parameter that is independent of the number of inner steps in Step-$K$ Nesterov, GPA provides a new foundation for re-designing algorithms for cross-regional training.

\subsection*{Acknowledgements}
We thank Anna Cai, Jiaming Cui, Runa Eschenhagen, Shagun Gupta, Yuchen Hao, Chien-Chin Huang, Minhui Huang, Dzmitry Huba, Vladimir Ivanov, Dominik Kallusky, Tsung-Hsien Lee, Vishal Nandavanam, Shangfu Peng, Vinay Rao, Junjie Wang, Gavin Zhang, Iris Zhang, Xin Zhang, and Chuanhao Zhuge for their discussions on the algorithm, support on experimentation in TorchTitan, and detailed feedback on the manuscript. We also thank Kristin Lauter, Maxim Naumov, Sandeep Parab, Joel Pobar, and Chunqiang Tang for their managerial support of this work. Finally, we thank Derek Shao, Emrah Seker and Rodrigo Paim from the Meta Superintelligence Labs Infrastructure team for supporting us with torchtitan conda packages in a timely manner to enable faster experimentation.

\bibliographystyle{assets/plainnat}
\bibliography{main}

@article{defazio2023when,
title={Optimal Linear Decay Learning Rate Schedules and Further Refinements},
author={Aaron Defazio and Ashok Cutkosky and Harsh Mehta and Konstantin Mishchenko},
  journal={arXiv preprint arXiv:2310.07831},
  year={2023}
}

@aticle{defazio2020mom,
title={Momentum via Primal Averaging: Theoretical Insights and Learning Rate Schedules for Non-Convex Optimization},
author={Aaron Defazio},
year={2020},
journal={arXiv preprint arXiv:2010.00406},
}

@article{polyak,
author = {Polyak, Boris T.},
year = {1990},
month = {01},
pages = {98-107},
title = {New stochastic approximation type procedures},
volume = {7},
journal = {Avtomatica i Telemekhanika}
}

@article{ruppert,
author = {Ruppert, David},
year = {1988},
month = {02},
pages = {},
title = {Efficient Estimations from a Slowly Convergent {Robbins}-{Monro} Process},
journal = {Technical Report, Cornell University}
}

@article{tao2018,
	author = {Tao, Wei and Pan, Zhisong and Wu, Gaowei and Tao, Qing},
	doi = {10.1109/TCYB.2018.2874332},
	journal = {IEEE Transactions on Cybernetics},
	month = {10},
	pages = {1-11},
	title = {Primal Averaging: A New Gradient Evaluation Step to Attain the Optimal Individual Convergence},
	volume = {PP},
	year = {2018}}

@inproceedings{lookahead,
	author = {Zhang, Michael and Lucas, James and Ba, Jimmy and Hinton, Geoffrey E.},
	booktitle = {Advances in Neural Information Processing Systems},
	editor = {H. Wallach and H. Larochelle and A. Beygelzimer and F. d\textquotesingle Alch\'{e}-Buc and E. Fox and R. Garnett},
	publisher = {Curran Associates, Inc.},
	title = {Lookahead Optimizer: $k$ steps forward, 1 step back},
	volume = {32},
	year = {2019}}

@Misc{Dahl2023AlgoPerf,
  title         = {{Benchmarking Neural Network Training Algorithms}},
  author        = {Dahl, George E. and Schneider, Frank and Nado, Zachary and Agarwal, Naman and Sastry, Chandramouli Shama and Hennig, Philipp and Medapati, Sourabh and Eschenhagen, Runa and Kasimbeg, Priya and Suo, Daniel and Bae, Juhan and Gilmer, Justin and Peirson, Abel L. and Khan, Bilal and Anil, Rohan and Rabbat, Mike and Krishnan, Shankar and Snider, Daniel and Amid, Ehsan and Chen, Kongtao and Maddison, Chris J. and Vasudev, Rakshith and Badura, Michal and Garg, Ankush and Mattson, Peter},
  year          = {2023},
  archiveprefix = {arXiv},
  eprint        = {2306.07179},
}

@article{lanaccel,
	author = {Lan, Guanghui},
	journal = {Mathematical Programming},
	number = {1},
	pages = {365--397},
	title = {An optimal method for stochastic composite optimization},
	volume = {133},
	year = {2012}}

@article{cesa2004generalization,
  title={On the generalization ability of on-line learning algorithms},
  author={Cesa-Bianchi, Nicolo and Conconi, Alex and Gentile, Claudio},
  journal={IEEE Transactions on Information Theory},
  volume={50},
  number={9},
  pages={2050--2057},
  year={2004},
  publisher={IEEE}
}

@inproceedings{kingma2014adam,
  title={Adam: a method for stochastic optimization},
  author={Kingma, Diederik P. and Ba, Jimmy},
  booktitle={International Conference on Learning Representations},
  year={2014}
}

@misc{sandler2023trainingtrajectories,
      title={Training trajectories, mini-batch losses and the curious role of the learning rate}, 
      author={Mark Sandler and Andrey Zhmoginov and Max Vladymyrov and Nolan Miller},
      year={2023},
      eprint={2301.02312},
      archivePrefix={arXiv},
      primaryClass={cs.LG},
      url={https://arxiv.org/abs/2301.02312}, 
}

@inproceedings{
loshchilov2018decoupled,
title={Decoupled Weight Decay Regularization},
author={Ilya Loshchilov and Frank Hutter},
booktitle={International Conference on Learning Representations},
year={2019},
url={https://openreview.net/forum?id=Bkg6RiCqY7},
}

@InProceedings{pmlr-v28-sutskever13,
  title = 	 {On the importance of initialization and momentum in deep learning},
  author = 	 {Sutskever, Ilya and Martens, James and Dahl, George and Hinton, Geoffrey},
  booktitle = 	 {Proceedings of the 30th International Conference on Machine Learning},
  year = 	 {2013},
  volume = 	 {28},
  series = 	 {Proceedings of Machine Learning Research},
  publisher =    {PMLR},
}

@article{diloco,
  title={{DiLoCo}: Distributed low-communication training of language models},
  author={Douillard, Arthur and Feng, Qixuan and Rusu, Andrei A. and Chhaparia, Rachita and Donchev, Yani and Kuncoro, Adhiguna and Ranzato, Marc'Aurelio and Szlam, Arthur and Shen, Jiajun},
  journal={arXiv preprint arXiv:2311.08105},
  year={2023}
}

@inproceedings{schedule-free,
	author = {Defazio, Aaron and Yang, Xingyu and Mehta, Harsh and Mishchenko, Konstantin and Khaled, Ahmed and Cutkosky, Ashok},
	booktitle = {Advances in Neural Information Processing Systems},
	editor = {A. Globerson and L. Mackey and D. Belgrave and A. Fan and U. Paquet and J. Tomczak and C. Zhang},
	pages = {9974--10007},
	publisher = {Curran Associates, Inc.},
	title = {The Road Less Scheduled},
	volume = {37},
	year = {2024}}

@inproceedings{pagliardini2025the,
title={The {AdEMAM}ix Optimizer: Better, Faster, Older},
author={Matteo Pagliardini and Pierre Ablin and David Grangier},
booktitle={The Thirteenth International Conference on Learning Representations},
year={2025},
url={https://openreview.net/forum?id=jj7b3p5kLY}
}

@article{liu2024asynchronous,
  title={Asynchronous {Local-SGD} Training for Language Modeling},
  author={Liu, Bo and Chhaparia, Rachita and Douillard, Arthur and Kale, Satyen and Rusu, Andrei A. and Shen, Jiajun and Szlam, Arthur and Ranzato, Marc'Aurelio},
  journal={arXiv preprint arXiv:2401.09135},
  year={2024}
}

@article{grattafiori2024llama3herdmodels,
      title={The {Llama} 3 Herd of Models}, 
      author={Llama Team, AI @ Meta},
      year={2024},
      journal={arXiv preprint arXiv:2407.21783},
}

@ARTICLE{adefazio-curvedgeom2019,
author = {Aaron Defazio},
title = {On the Curved Geometry of Accelerated Optimization},
journal = {Advances in Neural Information Processing Systems 33 (NIPS 2019)},
year = {2019}
}

@misc{jordan2024muon,
  author       = {Keller Jordan and Yuchen Jin and Vlado Boza and You Jiacheng and
                  Franz Cesista and Laker Newhouse and Jeremy Bernstein},
  title        = {Muon: An optimizer for hidden layers in neural networks},
  year         = {2024},
  url          = {https://kellerjordan.github.io/posts/muon/}
}

@InProceedings{pmlr-v80-gupta18a,
  title = 	 {Shampoo: Preconditioned Stochastic Tensor Optimization},
  author =       {Gupta, Vineet and Koren, Tomer and Singer, Yoram},
  booktitle = 	 {Proceedings of the 35th International Conference on Machine Learning},
  pages = 	 {1842--1850},
  year = 	 {2018},
  editor = 	 {Dy, Jennifer and Krause, Andreas},
  volume = 	 {80},
  series = 	 {Proceedings of Machine Learning Research},
  month = 	 {10--15 Jul},
  publisher =    {PMLR},
  url = 	 {https://proceedings.mlr.press/v80/gupta18a.html},
}

@article{shi2023distributed,
  title={A distributed data-parallel pytorch implementation of the distributed {Shampoo} optimizer for training neural networks at-scale},
  author={Shi, Hao-Jun Michael and Lee, Tsung-Hsien and Iwasaki, Shintaro and Gallego-Posada, Jose and Li, Zhijing and Rangadurai, Kaushik and Mudigere, Dheevatsa and Rabbat, Michael},
  journal={arXiv preprint arXiv:2309.06497},
  year={2023}
}

@article{eschenhagen2025purifying,
  title={Purifying {Shampoo}: Investigating {Shampoo's} Heuristics by Decomposing its Preconditioner},
  author={Eschenhagen, Runa and Defazio, Aaron and Lee, Tsung-Hsien and Turner, Richard E and Shi, Hao-Jun Michael},
  journal={arXiv preprint arXiv:2506.03595},
  year={2025}
}

@inproceedings{vyas2024soap,
    title={{SOAP}: Improving and Stabilizing Shampoo using {Adam} for Language Modeling},
    author={Nikhil Vyas and Depen Morwani and Rosie Zhao and Itai Shapira and David Brandfonbrener and Lucas Janson and Sham M. Kakade},
    booktitle={The Thirteenth International Conference on Learning Representations},
    year={2025},
    url={https://openreview.net/forum?id=IDxZhXrpNf}
}

@article{anil2020scalable,
  title={Scalable second order optimization for deep learning},
  author={Anil, Rohan and Gupta, Vineet and Koren, Tomer and Regan, Kevin and Singer, Yoram},
  journal={arXiv preprint arXiv:2002.09018},
  year={2020}
}

@inproceedings{douillard2025streaming,
  title={Streaming {DiLoCo} with overlapping communication: Towards a Distributed Free Lunch},
  journal={arXiv preprint arXiv:2501.18512},
  year={2025},
author={Arthur Douillard and Yani Donchev and J. Keith Rush and Satyen Kale and Zachary Charles and Gabriel Teston and Zachary Garrett and Jiajun Shen and Ross McIlroy and David Lacey and Alexandre Rame and Arthur Szlam and MarcAurelio Ranzato and Paul R. Barham},
booktitle={Second Conference on Language Modeling},
url={https://openreview.net/forum?id=yYk3zK0X6Q}
}

@inproceedings{charles2025communication,
      title={Communication-Efficient Language Model Training Scales Reliably and Robustly: Scaling Laws for {DiLoCo}}, 
      author={Zachary Charles and Gabriel Teston and Lucio M. Dery and J. Keith Rush and Nova Fallen and Zachary Garrett and Arthur Szlam and Arthur Douillard},
booktitle={The Thirty-ninth Annual Conference on Neural Information Processing Systems},
year={2025},
url={https://openreview.net/forum?id=X4SCxcgb3O}
}

@article{therien2025muloco,
  title={{MuLoCo}: {Muon} is a practical inner optimizer for {DiLoCo}},
  author={Th{\'e}rien, Benjamin and Huang, Xiaolong and Rish, Irina and Belilovsky, Eugene},
  journal={arXiv preprint arXiv:2505.23725},
  year={2025}
}

@article{large2024scalable,
  title={Scalable optimization in the modular norm},
  author={Large, Tim and Liu, Yang and Huh, Minyoung and Bahng, Hyojin and Isola, Phillip and Bernstein, Jeremy},
  journal={Advances in Neural Information Processing Systems},
  volume={37},
  pages={73501--73548},
  year={2024}
}

@InProceedings{pethick2025training,
  title={Training Deep Learning Models with Norm-Constrained {LMOs}},
  author={Pethick, Thomas and Xie, Wanyun and Antonakopoulos, Kimon and Zhu, Zhenyu and Silveti-Falls, Antonio and Cevher, Volkan},
  booktitle = 	 {Proceedings of the 42nd International Conference on Machine Learning},
  pages = 	 {49069--49104},
  year = 	 {2025},
  editor = 	 {Singh, Aarti and Fazel, Maryam and Hsu, Daniel and Lacoste-Julien, Simon and Berkenkamp, Felix and Maharaj, Tegan and Wagstaff, Kiri and Zhu, Jerry},
  volume = 	 {267},
  series = 	 {Proceedings of Machine Learning Research},
  publisher =    {PMLR},
  url = 	 {https://proceedings.mlr.press/v267/pethick25a.html},
}

@article{achiam2023gpt,
  title={{GPT}-4 Technical Report},
  author={Achiam, Josh and Adler, Steven and Agarwal, Sandhini and Ahmad, Lama and Akkaya, Ilge and Aleman, Florencia Leoni and Almeida, Diogo and Altenschmidt, Janko and Altman, Sam and Anadkat, Shyamal and others},
  journal={arXiv preprint arXiv:2303.08774},
  year={2023}
}

@article{liu2024deepseek,
  title={{DeepSeek-V3} Technical Report},
  author={Liu, Aixin and Feng, Bei and Xue, Bing and Wang, Bingxuan and Wu, Bochao and Lu, Chengda and Zhao, Chenggang and Deng, Chengqi and Zhang, Chenyu and Ruan, Chong and others},
  journal={arXiv preprint arXiv:2412.19437},
  year={2024}
}

@article{nichol2018reptile,
  title={On First-Order Meta-Learning Algorithms},
  author={Nichol, Alex and Schulman, John},
  journal={arXiv preprint arXiv:1803.02999},
  volume={2},
  number={3},
  pages={4},
  year={2018}
}

@inproceedings{finn2017model,
  title={Model-Agnostic Meta-Learning for Fast Adaptation of Deep Networks},
  author={Finn, Chelsea and Abbeel, Pieter and Levine, Sergey},
  booktitle={International conference on machine learning},
  pages={1126--1135},
  year={2017},
  organization={PMLR}
}

@article{robbins1951stochastic,
  title={A stochastic approximation method},
  author={Robbins, Herbert and Monro, Sutton},
  journal={The annals of mathematical statistics},
  pages={400--407},
  year={1951},
  publisher={JSTOR}
}

@inproceedings{yang2021tensor,
  title={Tensor Programs {IV}: Feature Learning in Infinite-Width Neural Networks},
  author={Yang, Greg and Hu, Edward J.},
  booktitle={International Conference on Machine Learning},
  pages={11727--11737},
  year={2021},
  organization={PMLR}
}

@inproceedings{yang2022tensor,
  title={Tensor programs {V}: Tuning large neural networks via zero-shot hyperparameter transfer},
  author={Yang, Greg and Hu, Edward J. and Babuschkin, Igor and Sidor, Szymon and Liu, Xiaodong and Farhi, David and Ryder, Nick and Pachocki, Jakub and Chen, Weizhu and Gao, Jianfeng},
 booktitle = {Advances in Neural Information Processing Systems},
 pages = {17084--17097},
 publisher = {Curran Associates, Inc.},
 title = {Tuning Large Neural Networks via Zero-Shot Hyperparameter Transfer},
 url = {https://proceedings.neurips.cc/paper_files/paper/2021/file/8df7c2e3c3c3be098ef7b382bd2c37ba-Paper.pdf},
 volume = {34},
 year = {2021}
}

@article{polyak1992acceleration,
  title={Acceleration of stochastic approximation by averaging},
  author={Polyak, Boris T. and Juditsky, Anatoli B.},
  journal={SIAM journal on control and optimization},
  volume={30},
  number={4},
  pages={838--855},
  year={1992},
  publisher={SIAM}
}

@article{morales2024exponential,
title={Exponential Moving Average of Weights in Deep Learning: Dynamics and Benefits},
author={Daniel Morales-Brotons and Thijs Vogels and Hadrien Hendrikx},
journal={Transactions on Machine Learning Research},
issn={2835-8856},
year={2024},
url={https://openreview.net/forum?id=2M9CUnYnBA},
}

@article{hoffmann2022training,
  title={Training compute-optimal large language models},
  author={Hoffmann, Jordan and Borgeaud, Sebastian and Mensch, Arthur and Buchatskaya, Elena and Cai, Trevor and Rutherford, Eliza and Casas, Diego de Las and Hendricks, Lisa Anne and Welbl, Johannes and Clark, Aidan and others},
  journal={arXiv preprint arXiv:2203.15556},
  year={2022}
}

@inproceedings{khaled2025understanding,
  title={Understanding Outer Optimizers in {Local SGD}: Learning Rates, Momentum, and Acceleration},
  author={Khaled, Ahmed and Kale, Satyen and Douillard, Arthur and Jin, Chi and Fergus, Rob and Zaheer, Manzil},
  booktitle={The Thirty-ninth Annual Conference on Neural Information Processing Systems},
  year={2025},
  url={https://openreview.net/forum?id=2VX79YLT9s},
}

@article{2019t5,
  title={Exploring the Limits of Transfer Learning with a Unified Text-to-Text Transformer},
  author={Raffel, Colin and Shazeer, Noam and Roberts, Adam and Lee, Katherine and Narang, Sharan and Matena, Michael and Zhou, Yanqi and Li, Wei and Liu, Peter J.},
  journal={Journal of machine learning research},
  volume={21},
  number={140},
  pages={1--67},
  year={2020}
}

@article{kallusky2025snoo,
  title={{SNOO}: Step-K {Nesterov} Outer Optimizer-The Surprising Effectiveness of Nesterov Momentum Applied to Pseudo-Gradients},
  author={Kallusky, Dominik and Rao, Vinay and Nandavanam, Vishal and Shi, Hao-Jun Michael},
  journal={arXiv preprint arXiv:2510.15830},
  year={2025}
}

@article{ziyin2020laprop,
  title={{LaProp}: Separating momentum and adaptivity in {Adam}},
  author={Ziyin, Liu and Wang, Zhikang T. and Ueda, Masahito},
  journal={arXiv preprint arXiv:2002.04839},
  year={2020}
}

@article{polyak1964some,
  title={Some methods of speeding up the convergence of iteration methods},
  author={Polyak, Boris T.},
  journal={Ussr computational mathematics and mathematical physics},
  volume={4},
  number={5},
  pages={1--17},
  year={1964},
  publisher={Elsevier}
}

@InProceedings{bernstein2024modular,
  title={Modular Duality in Deep Learning},
  author={Bernstein, Jeremy and Newhouse, Laker},
  journal={arXiv preprint arXiv:2410.21265},
  booktitle = 	 {Proceedings of the 42nd International Conference on Machine Learning},
  pages = 	 {3920--3930},
  year = 	 {2025},
  volume = 	 {267},
  series = 	 {Proceedings of Machine Learning Research},
  publisher =    {PMLR},
  url = 	 {https://proceedings.mlr.press/v267/bernstein25a.html},
}

@ARTICLE{7967721,
  author={Van Scoy, Bryan and Freeman, Randy A. and Lynch, Kevin M.},
  journal={IEEE Control Systems Letters}, 
  title={The Fastest Known Globally Convergent First-Order Method for Minimizing Strongly Convex Functions}, 
  year={2018},
  volume={2},
  number={1},
  pages={49-54},
  doi={10.1109/LCSYS.2017.2722406}}

\clearpage
\newpage
\beginappendix

\section{LLM Usage}
\label{app:llm-usage}
We used an internal AI assistant for revising the grammar and wording in the paper, and used Gemini Pro 2.5 to verify our proofs.

\section{Formulations of Polyak Momentum}
\label{app:polyak}

Similar to Nesterov momentum, classical or Polyak momentum also have different formulations that are commonly used in the community. The most commonly implemented formulation (which we call the \emph{modern formulation}) is given as:
\begin{align}
\label{eq:modern_form_polyak}
\begin{split}
b^{(t)} & = \mu b^{(t - 1)} + g(x^{(t)}; \xi^{(t)}),\\
x^{(t+1)} & = x^{(t)} - \gamma^{(t)} b^{(t)}.
\end{split}
\end{align}
The method accumulates a momentum buffer similar to Nesterov's modern formulation (\eqref{eq:modern_form}), but only updates the weights using $b^{(t)}$ as opposed to $\mu b^{(t)} + g(x^{(t)}; \xi^{(t)})$.

This formulation can be re-written in the \emph{heavy ball formulation}
\begin{equation}
\label{eq:heavy_ball_form_polyak}
x^{(t+1)} = x^{(t)} - \gamma^{(t)} b^{(t)} + \mu (x^{(t)} - x^{(t - 1)}),
\end{equation}
which is also equivalent to the \emph{primal averaging formulation} 
\citep{defazio2020mom}
\begin{align}
\label{eq:primal_averaging_form_polyak}
\begin{split}
z^{(t+1)} & = z^{(t)} - \gamma^{(t)} g(x^{(t)}; \xi^{(t)}),\\
x^{(t+1)} & = \mu x^{(t)} + \left(1-\mu \right)z^{(t+1)}.
\end{split}
\end{align}
{\bf Remarks.}
\begin{itemize}
    \item The LaProp algorithm \citep{ziyin2020laprop} uses the heavy ball formulation to motivate the generalization of momentum to preconditioned gradient methods by replacing the gradient $g(x^{(t)}; \xi^{(t)})$ with the search direction $d^{(t)}$ in \eqref{eq:modern_form_polyak}.
    \item The primal averaging formulations for Polyak momentum (\eqref{eq:primal_averaging_form_polyak}) and Nesterov momentum (\eqref{eq:primal_averaging_form}) differ in their inclusion of the $y^{(t)}$ interpolated sequence, which determines where the gradient is evaluated. This is also reflected in Sutskever's formulation (\eqref{eq:sutskever_form}).
    \item Polyak momentum can therefore be recovered by setting $\mu_y = 1$ in GPA (\eqref{eq:gpa}).
\end{itemize}

\begin{table}[hbt!]
\centering
\small
\caption{Summary of Polyak momentum formulations. Here, $\mu$ is the momentum parameter, $\gamma^{(t)}$ is the learning rate, and $g(\cdot; \xi^{(t)})$ denotes the stochastic gradient.}
\label{tab:momentum_formulations}
\begin{tabular}{p{3.5cm} p{6.5cm}}
\toprule
\textbf{Formulation} & \textbf{Update Equations} \\
\midrule
\makecell[l]{\textbf{Heavy Ball} \\ \citep{polyak1964some}} &
$x^{(t+1)} = x^{(t)} - \gamma^{(t)} b^{(t)} + \mu (x^{(t)} - x^{(t - 1)})$ \\
\midrule
\textbf{Modern (PyTorch/JAX)} &
$\begin{aligned}
b^{(t)} &= \mu b^{(t-1)} + g(x^{(t)}; \xi^{(t)}) \\
x^{(t+1)} &= x^{(t)} - \gamma^{(t)} b^{(t)}
\end{aligned}$ \\
\midrule
\makecell[l]{\textbf{Primal Averaging} \\ \citep{tao2018}} &
$\begin{aligned}
z^{(t+1)} &= z^{(t)} - \gamma^{(t)} g(x^{(t)}; \xi^{(t)}) \\
x^{(t+1)} &= \mu x^{(t)} + (1-\mu) z^{(t+1)}
\end{aligned}$ \\
\bottomrule
\end{tabular}
\end{table}

\section{Algorithmic Details}
\label{app:algorithm-details}

\subsection{Pseudocode for Single-Worker DiLoCo / Step-$K$ Nesterov}

We provide a complete description of non-distributed or single-worker DiLoCo (also known as Step-$K$ Nesterov Outer Optimizer) in Algorithm~\ref{alg:diloco}. 

\begin{algorithm}[H] 
\begin{algorithmic}[1]
    \Require Initial iterate $x^{(1)}$, inner learning rate schedule $\gamma^{(t)} > 0$, constant outer learning rate $\tilde{\gamma} > 0$, weight decay $\lambda \geq 0$, momentum parameter $\mu \in [0, 1)$, base optimizer $\baseopt$.
    \State $\tilde{x}^{(1)} = x^{(1)}$ \Comment{Initialize slow model weights.}
    \State $b^{(0)} = 0 \in \R^n$ \Comment{Initialize momentum buffer.}
    \For{step $t = 1, ..., T$}
        \State Sample mini-batch $\xi^{(t)}$
        \State $g^{(t)} \in \partial f(x^{(t)}; \xi^{(t)})$
        \State $d^{(t)} = \baseopt(g^{(t)})$ \Comment{Computes base optimizer's search direction.}
        \State $x^{(t+1)} = (1 - \gamma^{(t)} \lambda) x^{(t)} + \gamma^{(t)} d^{(t)}$ \Comment{Updates inner model weights (with weight decay).}
        \If{$t \text{ mod } H = 0$}
            \State $p^{(t)} = \tilde{x}^{(t)} - x^{(t+1)}$ \Comment{Pseudo-gradient computation.}
            \State $b^{(t+1)} = \mu b^{(t)} + p^{(t)}$ \Comment{Accumulates outer momentum.}
            \State $\tilde{x}^{(t + 1)} = \tilde{x}^{(t)} - \tilde{\gamma} \left[\mu b^{(t)} + p^{(t)} \right]$ \Comment{Nesterov-style parameter update.}
            \State $x^{(t+1)} = \tilde{x}^{(t+1)}$ \Comment{Re-initialize inner model weights.}
        \Else 
            \State $\tilde{x}^{(t + 1)} = \tilde{x}^{(t)}$
            \State $b^{(t + 1)} = b^{(t)}$
        \EndIf
    \EndFor\\
    \Return $\tilde{x}^{(T)}$
\end{algorithmic}
\caption{\label{alg:diloco} Single-Worker DiLoCo / Step-$K$ Nesterov}
\end{algorithm}

\subsection{Memory-Efficient Formulation of Generalized Primal Averaging}

The implementation of the original formulation of GPA in \eqref{eq:gpa} requires storing two additional copies of the model's parameters during the optimizer step. This is because the gradient computation occurs on the $y^{(t)}$ sequence, which is computed from the two other sequences $x^{(t)}$ and $z^{(t)}$. To avoid this additional model copy, we can store $y^{(t)}$ instead, and recover $x^{(t)}$ from $y^{(t)}$ and $z^{(t)}$ during evaluation time.

To see how this can be done, we define the \emph{memory-efficient formulation} of GPA as:
\begin{align}
\label{eq:memory-efficient}
\begin{split}
x^{(t)} & = \frac{1}{\mu_y} y^{(t)} + \left(1 - \frac{1}{\mu_y} \right) z^{(t)}, \\
y^{(t+1)} & = \mu_x y^{(t)} + (1 - \mu_x) z^{(t)} - (1 - \mu_x \mu_y) \gamma^{(t)} g(y^{(t)}; \xi^{(t)}), \\
z^{(t+1)} & =z^{(t)}-\gamma^{(t)} g(y^{(t)}; \xi^{(t)}). \\
\end{split}
\end{align}
This reformulation is valid only when $\mu_y > 0$. In the $y^{(t)}$ update, the first term can be interpreted as interpolating $y^{(t)}$ towards $z^{(t)}$. The second term is a correction term that applies a dampened update on $y^{(t)}$.

Note that this formulation does not require the computation of $x^{(t)}$ except when necessary. Therefore, our implementation enables a training and evaluation mode similar to neural network modules like batch normalization that enables us to compute $x^{(t)}$ from $y^{(t)}$ and vice-versa. Specifically, when switching from training to evaluation mode, we can compute $x^{(t)}$ from $y^{(t)}$ and $z^{(t)}$ by:
\begin{equation*}
    x^{(t)} = \frac{1}{\mu_y} y^{(t)} + \left(1 - \frac{1}{\mu_y} \right) z^{(t)}.
\end{equation*}
Similarly, when switching from evaluation to training mode, we can recover $y^{(t)}$ from $x^{(t)}$ and $z^{(t)}$ by:
\begin{equation*}
    y^{(t)} = \mu_y x^{(t)} + (1 - \mu_y) z^{(t)}.
\end{equation*}
A proof of the equivalence of these two formulations is provided in Appendix~\ref{app:proofs}. The complete pseudocode for arbitrary base optimizers are provided in Algorithm~\ref{alg:gpa-memory-efficient}.

\begin{algorithm}[H] 
\begin{algorithmic}[1]
    \Require Initial iterate $y^{(1)}$, learning rate schedule $\gamma^{(t)} > 0$, weight decay $\lambda \geq 0$, interpolation parameters $\mu_x, \mu_y \in [0, 1)$, base optimizer $\baseopt$.
    \State $z^{(1)} = y^{(1)}$
    \For{$t = 1, ..., T$}
        \State $g^{(t)} \in \partial f(y^{(t)}; \xi^{(t)})$
        \State $d^{(t)} = \baseopt(g^{(t)})$ 
        \State $y^{(t)} = \mu_x y^{(t)} + (1 - \mu_x) z^{(t)} + \gamma^{(t)} (1 - \mu_x \mu_y) (d^{(t)} + \lambda z^{(t)})$
        \State $z^{(t+1)} = (1 - \gamma^{(t)} \lambda) z^{(t)} - \gamma^{(t)} d^{(t)}$
    \EndFor\\
    \Return $x^{(T)} = \frac{1}{\mu_y} y^{(T)} + \left(1 - \frac{1}{\mu_y} \right) z^{(T)}$
\end{algorithmic}
\caption{\label{alg:gpa-memory-efficient}Memory-Efficient Generalized Primal Averaging (GPA)}
\end{algorithm}

\subsection{Memory Tradeoffs between DiLoCo and GPA}
\label{app:mem_tradeoffs_gpa}

The original GPA formulation stores two additional iterates, namely $x^{(t)}$, $y^{(t)}$ and $z^{(t)}$, incurring the same cost as DiLoCo. Our \emph{memory-efficient} implementation of GPA in \Cref{alg:gpa-memory-efficient}, stores $y^{(t)}$ in lieu of $x^{(t)}$ and $z^{(t)}$ and computes the iterate $x^{(t)}$ on-the-fly from $y^{(t)}$ and $z^{(t)}$. We present these tradeoffs in \Cref{tab:memory-compute-requirements}. It is worth noting that neither of the GPA formulations requires storing the momentum buffer $b^{(t)}$.

\begin{table}[H]
    \centering
    \caption{\textbf{Comparison of memory and compute requirements for AdamW, DiLoCo, and GPA variants}. While GPA-AdamW matches DiLoCo's memory overhead, the memory-efficient variant of GPA (GPA-M) reduces additional model copies to three without compromising model quality and convergence performance.}
    \begin{tabular}{l l l l}
    \hline
    Method & Optimizer States & Additional Copies & Additional FLOPs per  step\\
    \hline
    \rowcolor{colorlightgray} AdamW & $m^{(t)}$, $v^{(t)}$ & 2 & -- \\
    \rowcolor{colorlightgray} DiLoCo-AdamW & $m^{(t)}$, $v^{(t)}$, $\tilde{x}^{(t)}$, $b^{(t)}$ & 4 & $\mathcal{O}(n / H)$\\
    \rowcolor{colorblue} GPA-AdamW & $m^{(t)}$, $v^{(t)}$, $y^{(t)}$, $z^{(t)}$ & 4 & $\mathcal{O}(n)$ \\
    \rowcolor{colorblue} \textbf{GPA-M-AdamW} & $m^{(t)}$, $v^{(t)}$, \textbf{$z^{(t)}$} & \textbf{3} & \textbf{$\mathcal{O}(n)$} \\
    \hline
    \end{tabular}
    \label{tab:memory-compute-requirements}
\end{table}

\subsection{Compatibility with Modular Norm Theory}

Recent work on Muon and similar methods has built on modular norm theory, which suggests that the design of optimization methods for deep learning should constrain the modular norm of the model parameters in order to enable hyperparameter transferability and bounded Lipschitz constants \citep{large2024scalable,jordan2024muon,pethick2025training}. Here, we argue that GPA, by definition, preserves these norm constraints. 

To see this, assume that $d^{(t)}$ is the search direction for a single parameter that is constrained with respect to some norm, i.e., $\| d^{(t)} \| \leq M$ for some constant $M \geq 0$. (Typically, we assume that it is the RMS-to-RMS norm or similar.) We can preserve these norm constraints on the iterates produced by GPA since:
\begin{align*}
\|y^{(t)}\| & \leq \mu_y \|x^{(t)}\| + (1 - \mu_y) \|z^{(t)}\|\\
\|z^{(t+1)}\| & \leq (1 - \lambda \gamma^{(t)})\|z^{(t)}\| + \gamma^{(t)} \|d^{(t)}\|\\
\|x^{(t+1)}\| & \leq \mu_x \|x^{(t)}\| + \left(1-\mu_x \right) \|z^{(t+1)}\|.
\end{align*}
Since $\mu_x, \mu_y \in [0, 1]$, we can see that if $\max\left\{\|x^{(t)}\|, \|y^{(t)}\|, \|z^{(t)}\| \right\} \leq M'$ for $M' \geq 0$, then 
$$\max\left\{\|x^{(t + 1)}\|, \|y^{(t + 1)}\|, \|z^{(t + 1)}\| \right\} \leq (1 - \lambda \gamma^{(t)}) M' + \gamma^{(t)} M,$$ 
which is the same bound that we would obtain for the base optimizer.

\section{Proofs}
\label{app:proofs}

\subsection{Equivalence Between Nesterov's Formulations}

\begin{prop}
    Given fixed learning rates $\gamma_{\primal}, \gamma_{\modern} > 0$, Nesterov's primal averaging formulation (\eqref{eq:primal_averaging_form}) is equivalent to Nesterov's modern formulation (\eqref{eq:modern_form}) in the sense that 
    \begin{equation}
        \label{eq:nesterov-equalities}
        y_{\primal}^{(t)} = x_{\modern}^{(t)} ~~~ \text{and} ~~~ b_{\modern}^{(t)}=\frac{1}{\left(1 - \mu \right)\gamma_{\primal}}\left(x_{\primal}^{(t)} - x_{\primal}^{(t + 1)}\right),
    \end{equation}
    when $\mu_{\primal} = \mu_{\modern} = \mu$ and $\left(1 - \mu \right) \gamma_{\primal} = \gamma_{\modern}$.
\end{prop}

\begin{proof}
    We can prove this by induction. For simplicity of notation, we will use $x_m = x_{\modern}$ and $x_p = x_{\primal}$ and similar for all variables.
    
    For the base case, note that the initializations $z_p^{(1)} = x_p^{(1)} = x_m^{(1)}$ are equal. Therefore, 
    \begin{equation}
        \label{eq:y-primal-init}
        y_p^{(1)} = \mu x_p^{(1)} + (1 - \mu) z_p^{(1)} = x_m^{(1)},
    \end{equation}
    as desired. In addition, since $b_m^{(1)} = \mu b_m^{(0)} + g(x_m^{(1)}; \xi^{(1)}) = g(x_m^{(1)})$, we can see that:
    \begin{align*}
        x_p^{(1)} - x_p^{(2)} & = (1 - \mu) x_p^{(1)} - (1 - \mu) z_p^{(1)} \\
        & = (1 - \mu) (x_p^{(1)} - z_p^{(2)}) \\
        & = (1 - \mu) (x_p^{(1)} - z_p^{(1)} + \gamma_p g(y_p^{(1)}; \xi^{(1)})) \\
        & = (1 - \mu) \gamma_p g(y_p^{(1)}; \xi^{(1)}).
    \end{align*}
    The base case for the momentum buffer $b_m^{(1)}$ follows from rearranging the equation with \eqref{eq:y-primal-init} and observing that $b_m^{(1)} = \mu b_m^{(0)} + g(x_m^{(1)}; \xi^{(1)}) = g(x_m^{(1)}; \xi^{(1)})$.

    For the inductive step, assume that \eqref{eq:nesterov-equalities} holds for $t$. Then from the inductive hypothesis, we can show that:
    \begin{align}
        x_m^{(t + 1)} & = x_m^{(t)} - \gamma_m [\mu b_m^{(t)} + g(x_m^{(t)}; \xi^{(t)})] \nonumber \\
        & = y_p^{(t)} - (1 - \mu) \gamma_p \left[\mu \left( \frac{1}{(1 - \mu) \gamma_p} (x_p^{(t)} - x_p^{(t + 1)}) \right) + g(y_p^{(t)}; \xi^{(t)}) \right] \nonumber \\
        & = y_p^{(t)} - \mu (x_p^{(t)} - x_p^{(t + 1)}) - (1 - \mu) \gamma g(y_p^{(t)}; \xi^{(t)}). \label{eq:x_m-1}
    \end{align}
    From the primal averaging form in \eqref{eq:primal_averaging_form}, we can derive that:
    \begin{align}
        x_p^{(t + 1)} & = \mu x_p^{(t)} + (1 - \mu) z_p^{(t + 1)} \nonumber \\
        & = \mu x_p^{(t)} + (1 - \mu) (z_p^{(t)} - \gamma_p g(y_p^{(t)}; \xi^{(t)}) \nonumber \\
        & = y_p^{(t)} - (1 - \mu) \gamma_p g(y_p^{(t)}; \xi^{(t)}). \label{eq:x_p-1}
    \end{align}
    Rearranging \eqref{eq:x_p-1}, we get that:
    \begin{equation}
        \label{eq:y_p-x_p^(t+1)}
        y_p^{(t)} - x_p^{(t + 1)} = (1 - \mu) \gamma_p g(y_p^{(t)}; \xi^{(t)}).
    \end{equation}
    Plugging in \eqref{eq:y_p-x_p^(t+1)} into \eqref{eq:x_m-1}, we obtain:
    \begin{equation}
        \label{eq:x_m-2}
        x_m^{(t + 1)} = y_p^{(t)} - \mu (x_p^{(t)} - x_p^{(t + 1)}) - (y_p^{(t)} - x_p^{(t + 1)})
        = (1 + \mu) x_p^{(t + 1)} - \mu x_p^{(t)}.
    \end{equation}
    Finally, since $x_p^{(t + 1)} = \mu x_p^{(t)} + (1 - \mu) z_p^{(t)}$, $(1 - \mu) z_p^{(t + 1)} = x_p^{(t + 1)} - \mu x_p^{(t)}$. Therefore, to see $x_m^{(t + 1)}$'s equivalence to $y_p^{(t + 1)}$,
    \begin{align}
        y_p^{(t + 1)} & = \mu x_p^{(t + 1)} + (1 - \mu) z_p^{(t + 1)} \nonumber \\
        & = \mu x_p^{(t + 1)} + x_p^{(t + 1)} - \mu x_p^{(t)} \nonumber \\
        & = (1 + \mu) x_p^{(t + 1)} - \mu x_p^{(t)}. \label{eq:y_p-1}
    \end{align}
    Combining equations \ref{eq:x_m-2} and \ref{eq:y_p-1} gives the result.

    To prove that $b_m^{(t + 1)} = \frac{1}{(1 - \mu) \gamma_p} (x_p^{(t + 1)} - x_p^{(t + 2)})$, note that:
    \begin{equation}
        \label{eq:b_m-1}
        b_m^{(t + 1)} = \mu b_m^{(t)} + g(x_m^{(t + 1)}; \xi^{(t + 1)}) = \frac{\mu}{(1 - \mu) \gamma_p} (x_p^{(t)} - x_p^{(t + 1)}) + g(y_p^{(t + 1)}; \xi^{(t + 1)}).
    \end{equation}
    To get an expression for $x_p^{(t + 1)} - x_p^{(t + 2)}$, note that:
    \begin{align}
        x_p^{(t + 2)} & = \mu x_p^{(t + 1)} + (1 - \mu) (z_p^{(t + 1)} - \gamma_p g(y_p^{(t + 1)}; \xi^{(t + 1)})) \nonumber \\
        & = (\mu x_p^{(t + 1)} + (1 - \mu) z_p^{(t + 1)}) - (1 - \mu) \gamma_p g(y_p^{(t + 1)}; \xi^{(t + 1)}) \nonumber \\
        & = y_p^{(t + 1)} - (1 - \mu) \gamma_p g(y_p^{(t + 1)}; \xi^{(t + 1)}) \nonumber \\
        & = ((1 + \mu) x_p^{(t + 1)} - \mu x_p^{(t)}) - (1 - \mu) \gamma_p g(y_p^{(t + 1)}; \xi^{(t + 1)}) , \label{eq:x_p-2}
    \end{align}
    where \eqref{eq:x_p-2} follows from \eqref{eq:y_p-1}. Therefore, plugging-in \eqref{eq:x_p-2} into $x_p^{(t + 1)} - x_p^{(t + 2)}$ gives:
    \begin{equation}
        x_p^{(t + 1)} - x_p^{(t + 2)} = - \mu (x_p^{(t + 1)} - x_p^{(t)}) + (1 - \mu) \gamma_p g(y_p^{(t + 1)}; \xi^{(t + 1)}).
    \end{equation}
    The result follows from expanding \eqref{eq:b_m-1} as:
    \begin{align*}
        b_m^{(t + 1)} & = \frac{1}{(1 - \mu) \gamma_p} \left[- \mu (x_p^{(t + 1)} - x_p^{(t)}) + (1 - \mu) \gamma_p g(y_p^{(t + 1)}; \xi^{(t + 1)}) \right] \\
        & = \frac{1}{(1 - \mu) \gamma_p} (x_p^{(t + 1)} - x_p^{(t + 2)}).
    \end{align*}
\end{proof}

\subsection{Equivalence Between Generalized Primal Averaging Formulations}

\begin{prop}
    Let $\mu_y > 0$. Then GPA (\eqref{eq:gpa}) is equivalent to the memory-efficient formulation (\eqref{eq:memory-efficient}).
\end{prop}

\begin{proof}
Note that it is sufficient to show that:
\begin{align}
    x^{(t)} & = \frac{1}{\mu_y} y^{(t)} + \left(1 - \frac{1}{\mu_y} \right) z^{(t)}, \label{eq:memory-efficient-x} \\
    y^{(t + 1)} & = \mu_x y^{(t)} + (1 - \mu_x) z^{(t)} - (1 - \mu_x \mu_y) \gamma^{(t)} g(y^{(t)}; \xi^{(t)}). \label{eq:memory-efficient-y}
\end{align}
To prove \eqref{eq:memory-efficient-x}, note that we can re-write $x^{(t)}$ as a function of $y^{(t)}$ and $z^{(t)}$, i.e., since
\begin{equation*}
    y^{(t)} = \mu_y x^{(t)} + (1 - \mu_y) z^{(t)}
\end{equation*}
and $\mu_y > 0$, we have that
\begin{equation*}
    x^{(t)} = \frac{1}{\mu_y} y^{(t)} + \frac{1}{\mu_y} (\mu_y - 1) z^{(t)} = \frac{1}{\mu_y} y^{(t)} + \left(1 - \frac{1}{\mu_y} \right) z^{(t)}.
\end{equation*}

To prove \eqref{eq:memory-efficient-x}, we can re-write \eqref{eq:memory-efficient-x} as
\begin{equation}
    \label{eq:mu-y-x-1}
    \mu_y x^{(t + 1)} = \mu_y z^{(t + 1)} + (y^{(t + 1)} - z^{(t + 1)}) = y^{(t + 1)} - (1 - \mu_y) z^{(t + 1)}.
\end{equation}
Similarly, by plugging in the original $x^{(t + 1)}$ update, i.e., $x^{(t + 1)} = \mu_x x^{(t)} + (1 - \mu_x) z^{(t)}$, we also have:
\begin{equation}
    \label{eq:mu-y-x-2}
    \mu_y x^{(t + 1)} = \mu_y (\mu_x x^{(t)} + (1 - \mu_x) z^{(t)}) = \mu_x \mu_y x^{(t)} + (1 - \mu_x) \mu_y z^{(t + 1)}.
\end{equation}
Combining these two equalities in equations \ref{eq:mu-y-x-1} and \ref{eq:mu-y-x-2} and rearranging, we get:
\begin{equation}
    \label{eq:y_interp_x_z}
    y^{(t + 1)} = \mu_x \mu_y x^{(t)} + (1 - \mu_x \mu_y) z^{(t + 1)}.
\end{equation}
Plugging-in \eqref{eq:memory-efficient-x} and the update $z^{(t + 1)} = z^{(t)} - \gamma^{(t)} g(y^{(t)}; \xi^{(t)})$ from \eqref{eq:gpa} into \eqref{eq:y_interp_x_z}, we obtain:
\begin{align*}
    y^{(t + 1)} & = \mu_x \mu_y \left(\frac{1}{\mu_y} y^{(t)} + \left(1 - \frac{1}{\mu_y} \right) z^{(t)} \right) + (1 - \mu_x \mu_y) (z^{(t)} - \gamma^{(t)} g(y^{(t)}; \xi^{(t)})) \\
    & = \mu_x y^{(t)} + (1 - \mu_x) z^{(t)} - (1 - \mu_x \mu_y) \gamma^{(t)} g(y^{(t)}; \xi^{(t)}),
\end{align*}
as desired.
\end{proof}

\subsection{Convergence Bounds Based On Online-to-Batch Theory}

Our proofs similarly rely on the online-to-batch conversion theory used in \cite{schedule-free}.

\begin{lem}
\label{lem:ema}
Suppose we define $w^{(t)}$ as the weighting:
\[
w^{(t)}=\begin{cases}
1 & \mbox{if } t = 1,\\
\left(1 - \mu_x\right)\mu_x^{-t+1} & \mbox{if } t > 1.
\end{cases}
\]
Then the model evaluation sequence $x^{(t)}$
is equivalent to the weighted average:
\[
x^{(t + 1)} = \frac{\sum_{i=1}^{t}w^{(i)}}{\sum_{i=1}^{t+1}w^{(i)}} x^{(t)} + \frac{w^{(t+1)}}{\sum_{i=1}^{(t+1)} w^{(i)}} z^{(t+1)} = \frac{w^{(1:t)}}{w^{(1:t+1)}} x^{(t)} + \frac{w^{(t + 1)}}{w^{(1:t+1)}} z^{(t + 1)},
\]
with
\[
w^{(1:t)}=\sum_{s=1}^{t} w^{(s)}=\mu_x^{-t+1}.
\]
Furthermore, $x^{(t)}$ can be expressed as the closed form expression:
\[
x^{(t)} = \mu_x^{t-1}\sum_{s=1}^{t} w^{(s)} z^{(s)}.
\]
\end{lem}

\begin{thm}
Let $F$ be a convex function, and assume that there exists a minimizer $x_*$ that minimizes $F$. Let $\xi^{(1)},\dots,\xi^{(T)}$ be a sequence of i.i.d. random variables. Suppose that we are given arbitrary updates $z^{(1)},\dots,z^{(T)}$ from a base optimizer within the Generalized Primal Averaging framework (Equation~\ref{eq:gpa}). Then for $\mu_x, \mu_y \in [0, 1)$ and average iterate $\bar{x}^{(T)} = \frac{1}{T} \sum_{t=1}^{T} x^{(t)}$, we have the bound
\begin{align*}
\mathbb{E}[F(\bar{x}^{(T)})-F(x_{*})] & \leq \frac{1}{T}\sum_{t=1}^{T}\mathbb{E}[\langle\nabla F(y^{(t)}),z^{(t)}-x_{*}\rangle]\\
 & \qquad +\frac{\mu_x}{1-\mu_x}\frac{1}{T}\mathbb{E}\left[F(x^{(1)})-F(x_*)\right]\\
 & \qquad -\frac{1}{1-\mu_y}\frac{1}{T}\sum_{t=1}^{T}\mathbb{E}[B_{F}(y^{(t)},x^{(t)})] -\frac{\mu_y}{1-\mu_y}\frac{1}{T}\sum_{t=1}^{T} \mathbb{E}[B_{F}(x^{(t)},y^{(t)})] \\
 & \qquad -\frac{\mu_x}{1-\mu_x}\frac{1}{T}\sum_{t=1}^{T}\mathbb{E}[B_{F}(x^{(t-1)},x^{(t)})].
\end{align*}
\end{thm}

\begin{proof}
We start with the same analysis as in the Schedule-Free work \citep{schedule-free}. Notice that by definition of $x^{(t)}$, it holds $w^{(1:t-1)}(x^{(t)} - x^{(t-1)}) = w^{(t)}(z^{(t)} - x^{(t)})$. Therefore,
\begin{align*}
    w^{(1:t)}F(x^{(t)}) &- w^{(1:t-1)}F(x^{(t-1)}) -w^{(t)}F(x_{*}) \\
    &= w^{(1:t-1)}(F(x^{(t)}) - F(x^{(t-1)})) + w^{(t)} (F(x^{(t)}) - F(x_*)) \\
    &= w^{(1:t-1)}(\langle \nabla F(x^{(t)}), x^{(t)} - x^{(t-1)}\rangle - B_F(x^{(t-1)}, x^{(t)})) + w^{(t)} (F(x^{(t)}) - F(x_*)) \\
    &= w^{(t)}\langle \nabla F(x^{(t)}), z^{(t)} - x^{(t)}\rangle - w^{(1:t-1)} B_F(x^{(t-1)}, x^{(t)}) + w^{(t)} (F(x^{(t)}) - F(x_*)).
\end{align*}
Next, we observe that by definition of $y^{(t)}$, it holds $z^{(t)} - y^{(t)} = \frac{\mu_y}{1 - \mu_y}(y^{(t)} - x^{(t)})$, and, thus,
\begin{align*}
    \langle &\nabla  F(x^{(t)}), z^{(t)} - x^{(t)}\rangle \\
    &= \langle \nabla F(x^{(t)}) - \nabla F(y^{(t)}), z^{(t)} - y^{(t)}\rangle + \langle \nabla F(y^{(t)}), z^{(t)} - y^{(t)}\rangle \\
    &\quad + \langle \nabla F(x^{(t)}), y^{(t)} - x^{(t)}\rangle \\
    &= \frac{\mu_y}{1 - \mu_y}\langle \nabla F(x^{(t)}) - \nabla F(y^{(t)}), y^{(t)} - x^{(t)}\rangle + F(x_{*}) - F(y^{(t)}) - B_F(x_{*}, y^{(t)}) + \langle \nabla F(y^{(t)}), z^{(t)} - x_*\rangle \\
    &\quad + F(y^{(t)}) - F(x^{(t)}) - B_F(y^{(t)}, x^{(t)}) \\
    &\le -\frac{\mu_y}{1 - \mu_y}(B_F(x^{(t)}, y^{(t)}) + B_F(y^{(t)}, x^{(t)})) + F(x_{*}) - F(x^{(t)}) - B_F(y^{(t)}, x^{(t)}) + \langle \nabla F(y^{(t)}), z^{(t)} - x_*\rangle \\
    &= -\frac{\mu_y}{1 - \mu_y}B_F(x^{(t)}, y^{(t)}) - \frac{1}{1 - \mu_y} B_F(y^{(t)}, x^{(t)}) + F(x_{*}) - F(x^{(t)}) + \langle \nabla F(y^{(t)}), z^{(t)} - x_*\rangle,
\end{align*}
where the inequality step used $-B_F(x_*, y^{(t)})\le 0$, which follows from convexity of $F$. 
Plugging this back, we obtain
\begin{align}
    w^{(1:t)}F(x^{(t)}) &- w^{(1:t-1)}F(x^{(t-1)}) -w^{(t)}F(x_{*}) \notag\\
    &\le -w^{(t)}\frac{\mu_y}{1 - \mu_y}B_F(x^{(t)}, y^{(t)}) - \frac{w^{(t)}}{1 - \mu_y} B_F(y^{(t)}, x^{(t)}) + w^{(t)}( F(x_{*}) - F(x^{(t)}) ) \notag\\
    &\quad + w^{(t)} \langle \nabla F(y^{(t)}), z^{(t)} - x_*\rangle - w^{(1:t-1)} B_F(x^{(t-1)}, x^{(t)}) + w^{(t)} (F(x^{(t)}) - F(x_*)) \notag\\
    & = w^{(t)} \langle\nabla F(y^{(t)}),z^{(t)} - x_{*}\rangle
    - \frac{w^{(t)}}{1 - \mu_y} B_{F}(y^{(t)}, x^{(t)}) \notag \\
    & \quad - \frac{w^{(t)} \mu_y}{1 - \mu_y} B_{F}(x^{(t)}, y^{(t)}) 
    - w^{(1:t-1)}B_{F}(x^{(t-1)},x^{(t)}). \label{eq:onlinetobatch}
\end{align}

We may adapt this bound to our setting by using an exponentially increasing
weighting sequence, given by Lemma \ref{lem:ema}. Using those weights,
we have simplified expressions for the following quantities:
\begin{align*}
\frac{w^{(1:t)}}{w^{(t)}} & = \frac{\mu_x^{-t+1}}{\left(1-\mu_x\right)\mu_x^{-t+1}}=\frac{1}{1-\mu_x}, \\
\frac{w^{(1:t-1)}}{w^{(t)}} & = \frac{\mu_x^{-(t-1)+1}}{\left(1-\mu_x\right)\mu_x^{-t+1}}=\frac{\mu_x}{1-\mu_x},
\end{align*}
with a special case for the first iterate $\frac{w^{(1:1)}}{w^{(1)}} = 1$ and $\frac{w^{(1:t-1)}}{w^{(1)}} = 0$.

To obtain an average regret bound, we divide Equation \ref{eq:onlinetobatch}
by $w^{(t)}$, take expectation, and sum from $1$ to $T$. The left-hand side is a telescoping
sum, which we can simplify as follows:
\begin{align*}
 &\sum_{t=1}^{T}\left[\frac{w^{(1:t)}}{w^{(t)}}\mathbb{E}[F(x^{(t)})] -\frac{w^{(1:t-1)}}{w^{(t)}}\mathbb{E}[F(x^{(t-1)})]\right]-TF(x_{*}) \\
 & =F(x^{(1)})-\frac{w^{(1:1)}}{w^{(2)}}F(x^{(1)}) + \frac{1}{1-\mu_x}\sum_{t=2}^{T}\mathbb{E}[F(x^{(t)})] -\frac{\mu_x}{1-\mu_x}\sum_{t=2}^{T-1}\mathbb{E}[F(x^{(t)})] -TF(x_{*})\\
 & =F(x^{(1)})-\frac{1}{\left(1-\mu_x\right)\mu_x^{-1}}F(x^{(1)})+\frac{1}{1-\mu_x}\mathbb{E}[F(x^{(T)})] + \sum_{t=2}^{T-1}\left(\frac{1}{1-\mu_x}-\frac{\mu_x}{1-\mu_x}\right)\mathbb{E}[F(x^{(t)})] - TF(x_{*})\\
 & =F(x^{(1)})-\frac{\mu_x}{1-\mu_x}F(x^{(1)})+\frac{1}{1-\mu_x}\mathbb{E}[F(x^{(T)})] + \sum_{t=2}^{T-1}\left(\frac{1}{1-\mu_x}-\frac{\mu_x}{1-\mu_x}\right)\mathbb{E}[F(x^{(t)})] - TF(x_{*})\\
 & =-\frac{\mu_x}{1-\mu_x}F(x^{(1)})+\frac{\mu_x}{1-\mu_x}\mathbb{E}[F(x^{(T)})] + \sum_{t=1}^{T}\mathbb{E}[F(x^{(t)})] - TF(x_{*}).
\end{align*}
Plugging-in this simplified expression, moving the extra $F(x^{(1)})-F(x^{(t)})$
term to the right-hand side, and simplifying gives:
\begin{align*}
\sum_{t=1}^{T}\mathbb{E}\left[F(x^{(t)})-F(x_{*})\right] & \leq\sum_{t=1}^{T}\mathbb{E}[\langle\nabla F(y^{(t)}),z^{(t)}-x_{*}\rangle] +\frac{\mu_x}{1-\mu_x}\mathbb{E}\left[F(x^{(1)})-F(x^{(T)})\right]\\
 & -\frac{1}{1-\mu_y}\sum_{t=1}^{T}\mathbb{E}[B_{F}(y^{(t)},x^{(t)})] -\frac{\mu_y}{1-\mu_y}\sum_{t=1}^{T}\mathbb{E}[B_{F}(x^{(t)},y^{(t)})]\\
 & -\frac{\mu_x}{1-\mu_x}\sum_{t=1}^{T}\mathbb{E}[B_{F}(x^{(t-1)},x^{(t)})].
\end{align*}
We get a bound on the average iterate $\bar{x}_{T}=\sum_{t=1}^{T}x^{(t)}$
by dividing by $T$ and applying Jensen's inequality:
\begin{align*}
\mathbb{E}[F(\bar{x}_{T})-F(x_{*})] & \leq\frac{1}{T}\mathbb{E}\sum_{t=1}^{T}\langle\nabla F(y^{(t)}),z^{(t)}-x_{*}\rangle +\frac{\mu_x}{1-\mu_x}\frac{1}{T}\mathbb{E}\left[F(x^{(1)})-F(x^{(T)})\right]\\
 & -\frac{1}{1-\mu_y}\frac{1}{T}\mathbb{E}\sum_{t=1}^{T}B_{F}(y^{(t)},x^{(t)})-\frac{\mu_y}{1-\mu_y}\frac{1}{T}\mathbb{E}\sum_{t=1}^{T}B_{F}(x^{(t)},y^{(t)})\\
 & -\frac{\mu_x}{1-\mu_x}\frac{1}{T}\mathbb{E}\sum_{t=1}^{T}B_{F}(x^{(t-1)},x^{(t)}).
\end{align*}
Finally, we use $F(x_*)\le F(x^{(T)})$ to get the claimed bound.
\end{proof}

\begin{corollary}
    Assume that the base optimizer has regret guarantees $\sum_{t=1}^{T}\mathbb{E}[\langle\nabla F(y^{(t)}),z^{(t)}-x_{*}\rangle] = \mathcal{O}(\sqrt{T})$. Then:
    \begin{align*}
    \mathbb{E}[F(\bar{x}^{(T)})-F(x_{*})]
    = \mathcal{O} \left( \frac{1}{\sqrt{T}} \right).
\end{align*}
\end{corollary}

\begin{proof}
    Note that we can upper bound the inequality in Theorem~\ref{thm:genbregman-gpa} by ignoring the negative Bregman divergence terms, i.e., 
    $$\mathbb{E}[F(\bar{x}^{(T)})-F(x_{*})] \leq \frac{1}{T}\sum_{t=1}^{T}\mathbb{E}[\langle\nabla F(y^{(t)}),z^{(t)}-x_{*}\rangle] +\frac{\mu_x}{1-\mu_x}\frac{1}{T}\mathbb{E}\left[F(x^{(1)})-F(x_*)\right].$$
    The result follows from noting that the first term is $\mathcal{O}(1 / \sqrt{T})$ and the second term is $\mathcal{O}(1 / T)$.
\end{proof}

\section{Experimental Details}
\label{app:expmt_details}

\subsection{Comparison Between GPA and Nesterov}
\label{app:expmt_details_gpacoeffs}

In order to validate that DiLoCo's performance can only be matched or improved upon with decoupled interpolation constants in GPA, we test the case where $\mu_x = \mu_y$, which corresponds to Nesterov's primal averaging formulation in \Cref{eq:primal_averaging_form}. Here, we apply the same heuristic for $\mu_x = \mu^{1 / H}$ also to $\mu_y$. In \Cref{fig:gpacoeffs_comm_int8}, we observe that coupling the interpolation constants is sub-optimal, and decoupling these coefficients is necessary for the best performance from GPA. This matches the poor performance of single-worker DiLoCo with a single inner step.

\begin{figure}[H]
\centering \includegraphics[width=0.55\linewidth]{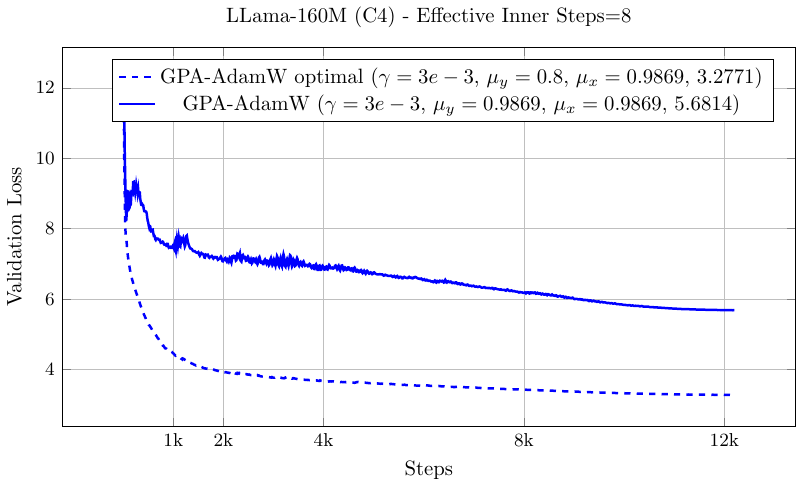}
  \caption{Comparison between Nesterov's primal averaging formulation with coupled constants $\mu_x = \mu_y$ and GPA with decoupled constants.}
  \label{fig:gpacoeffs_comm_int8}
\end{figure}


\subsection{Comparison of methods across number of inner steps for Llama-1B model} 
\label{app:expmt_details_consolidated_speedup_1B}

Figure~\ref{fig:consolidated_valloss_vs_comm_intervals_1B} and Figure~\ref{fig:bar_consolidated_valloss_vs_comm_intervals_1B} show the performance of the three methods on Llama-1B model as the effective number of inner steps is varied. Consistent with earlier findings on LLama-160M, GPA obtains speedups of upto 10.13\% over AdamW.

\begin{figure} [H] 
    \begin{subfigure}{0.54\textwidth}
        \includegraphics[width=\linewidth]{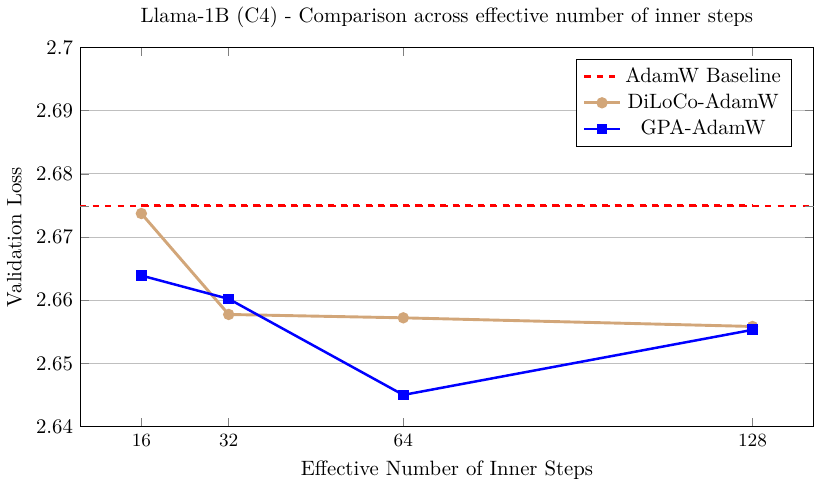}
        \caption{Both GPA and single-worker DiLoCo, when using AdamW as their base optimizer, outperform the tuned AdamW baseline for training a 1B parameter Llama model.}
        \label{fig:consolidated_valloss_vs_comm_intervals_1B}
    \end{subfigure}
    \quad
    \begin{subfigure}{0.44\textwidth}
        \begin{center}
            \includegraphics[width=\textwidth]{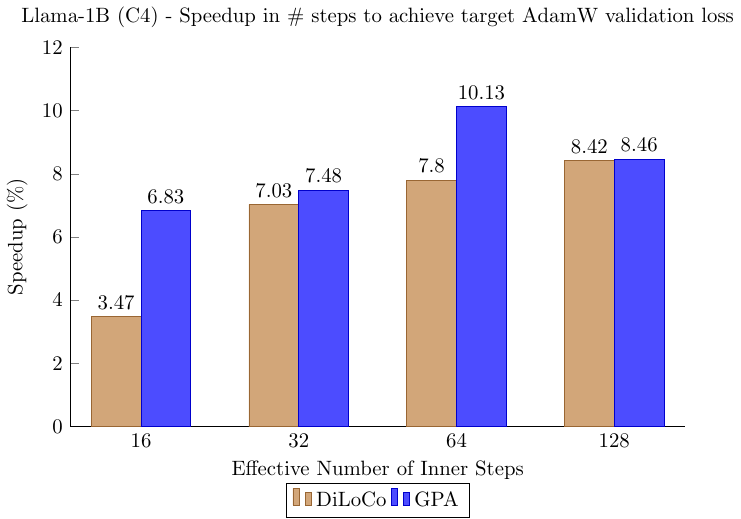}
        \end{center}
        \caption{Speedup achieved by single-worker DiLoCo and GPA measured in terms of reduction in number of steps required to attain the final validation loss obtained by AdamW, across different effective numbers of inner steps. GPA attains the highest speedup of 10.13\% when the effective inner steps is equal to 64.}
        \label{fig:bar_consolidated_valloss_vs_comm_intervals_1B}
    \end{subfigure}
    \caption{Comparison of validation loss and speedup for AdamW, single-worker DiLoCo, and GPA.}
    \label{fig:consolidated_plots_1B}
\end{figure}

\subsection{Additional Validation Loss Curves for Llama-160M and 1B models} 
\label{app:expmt_details_addnlplots}

In Figure~\ref{fig:comm_int8_valloss_vs_steps} and Figure~\ref{fig:comm_int64_128_valloss_vs_steps}, we provide additional validation loss curves for Llama-160M and Llama-1B models, where the effective number of inner steps equals 64 and 128, respectively. Consistent with earlier findings, we find that GPA outperforms baselines in all the cases, except when the number of inner steps equal to 128 for 1B model, where DiLoCo performs better.

\begin{figure}[H] 
  \includegraphics[width=0.49\linewidth]{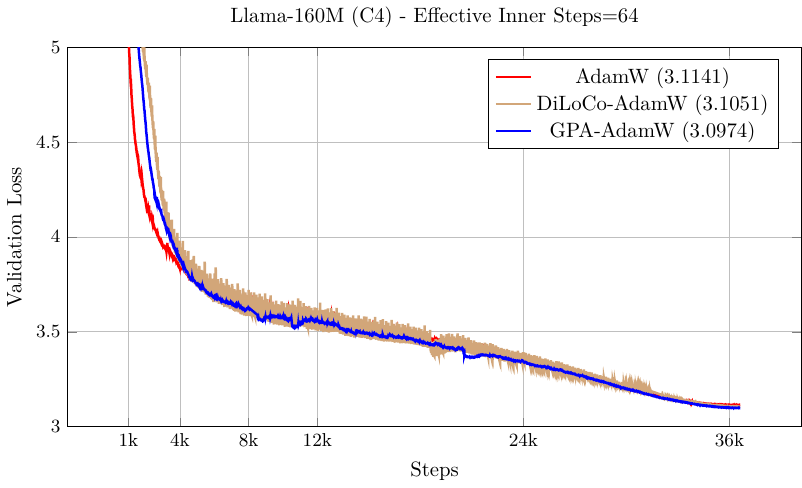}
  \qquad
  \includegraphics[width=0.49\linewidth]{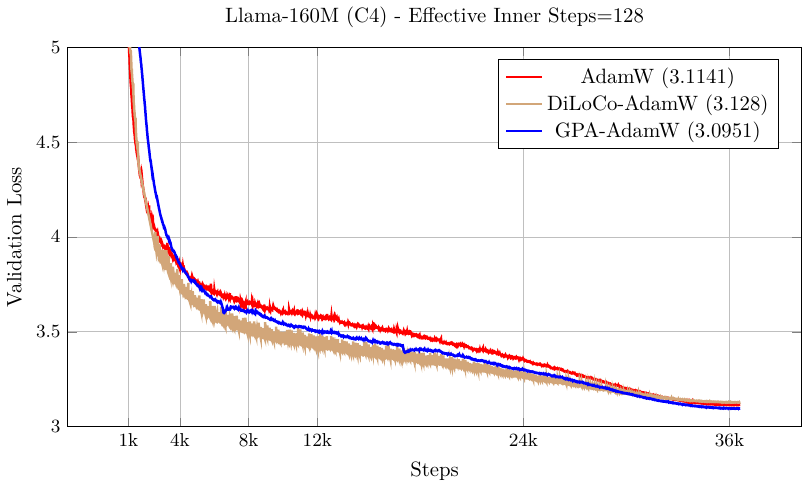}
  \caption{Validation loss versus steps for GPA, DiLoCo and AdamW when the effective number of inner steps equals $H = 64$ (left) and $H = 128$ (right).}
  \label{fig:comm_int8_valloss_vs_steps}
\end{figure}

\begin{figure}[H] 
  \includegraphics[width=0.49\linewidth]{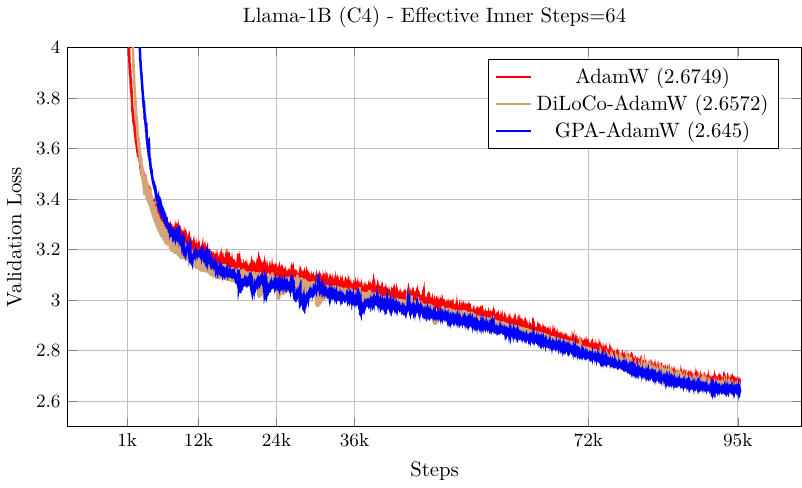}
  \qquad
  \includegraphics[width=0.49\linewidth]{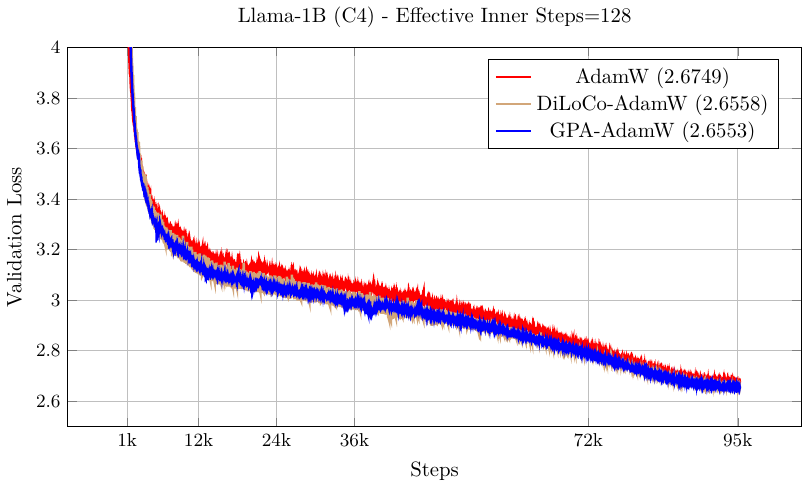}
  \caption{Validation loss versus steps for GPA, DiLoCo and AdamW when the effective number of inner steps equals $H = 64$ (left) and $H = 128$ (right).}
  \label{fig:comm_int64_128_valloss_vs_steps}
\end{figure}

\subsection{Additional Validation Loss Curves for ImageNet ViT model training}
\label{app:expts_vit_additional}
\begin{figure}[hbt!]
    \centering
    \includegraphics[width=0.48\linewidth]{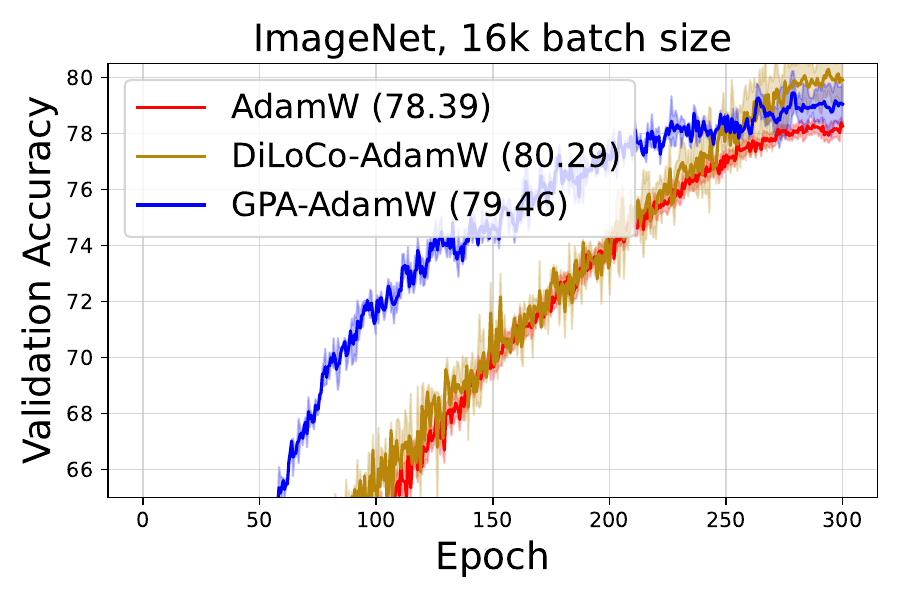}
    \quad
    \includegraphics[width=0.48\linewidth]{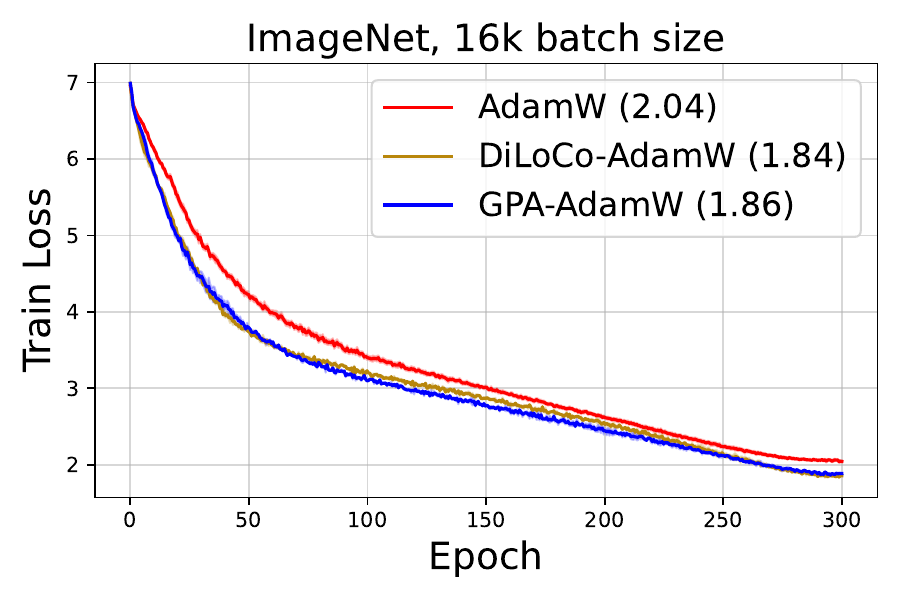}
    \caption{Comparison of AdamW and GPA on ImageNet ViT-S/16 from \texttt{timm} with data augmentations using a {\bf batch size of 16,384 samples}.}
    \label{fig:vit_experiments_16k}
\end{figure}

\subsection{Hyperparameter Sweeps for Llama-160M}
\label{app:expmt_details_hps}

\noindent {\bf Training setup.} We evaluate AdamW, DiLoCo-AdamW, and GPA-AdamW by pre-training the 160 million parameter Llama 3 model on the C4 dataset from scratch \citep{2019t5}. We follow an overtrained setup of 9.6 billion tokens, which is roughly three times the Chinchilla-optimal token budget \citep{hoffmann2022training}. All of our experiments are conducted on a single machine equipped with two GB200 GPUs (184GB memory). We used a batch size of 128 sequences with a sequence length of 2048 tokens, resulting in a total batch size of about 262,144 tokens. A summary of the hyperparameter sweeps is provided in Table~\ref{tab:hyperparameter_sweeps}.

{\bf Hyperparameter tuning strategy.}
\begin{itemize}
    \item {\bf AdamW}: We found that we could significantly improve the performance of AdamW by tuning the beta hyperparameters. In our study, we tune $(\beta_1, \beta_2)$ on a fine granular grid in the range of $(0.9, 0.999)$ and $\epsilon = 10^{-8}$, and sweep the learning rate from $7\cdot 10^{-4}$ to $1\cdot 10^{-2}$. Since tuning all hyperparameters simultaneously is computationally prohibitive, our sweeps are organized into multiple phases, with each phase consisting of analyzing a particular hyperparameter while keeping others fixed. We follow this strategy for all the methods. 
    \item {\bf DiLoCo-AdamW}: We provide the same search strategy for beta hyperparameters as done for AdamW. We additionally sweep the outer learning rate from $[0.35, 0.75, 1.0]$, the outer momentum from $[0.7, 0.75, 0.9, 1.0]$, and the number of inner steps from $[1, 128]$ with powers of 2.
    \item {\bf GPA-AdamW}: Similar to AdamW, we tune the beta hyperparameters for GPA following the same grid search. We additionally sweep $\mu_x$ based on the number of inner steps in DiLoCo (see Section \ref{sec:gpa-diloco}) and $\mu_y$ in $\{0.7, 0.8, 0.9, 0.95\}$. We follow the same search for learning rate values as other baselines.
    \item {\bf Schedule-Free-AdamW}: We use the same search strategy for beta hyperparameters as carried out for GPA. We sweep over the same range of learning rate and $\mu_{y}$ hyperparameters as used for GPA.
\end{itemize}

All runs employ a learning rate schedule with a linear warmup over the first 10\% of training, followed by cosine decay for the remainder of training (with the minimum learning rate factor is set to $0.0$). By default, we use gradient clipping with a clipping factor of $1.0$, except for GPA, where clipping can also be disabled. Weight decay is fixed at $0.1$. A summary of the hyperparameter sweeps is provided in Table~\ref{tab:hyperparameter_sweeps}.

Consistent with our tuning process, we provide a sensitivity analysis for each set of hyperparameters. In Figure~\ref{fig:hp_betas}, we examine the impact of different choices of $(\beta_1, \beta_2)$ for each of method. Figure~\ref{fig:hp_gpa_coeffs} analyzes the interpolation coefficients $\mu_{x}$ and $\mu_{y}$ in GPA. Lastly, in Figure~\ref{fig:hp_diloco}, we analyze the effects of varying the inner and outer learning rates, as well as the global momentum, for DiLoCo.

\begin{figure}[hbt!]
    \centering
    \begin{subfigure}{0.33\textwidth}
        \centering
        \includegraphics[width=\linewidth]{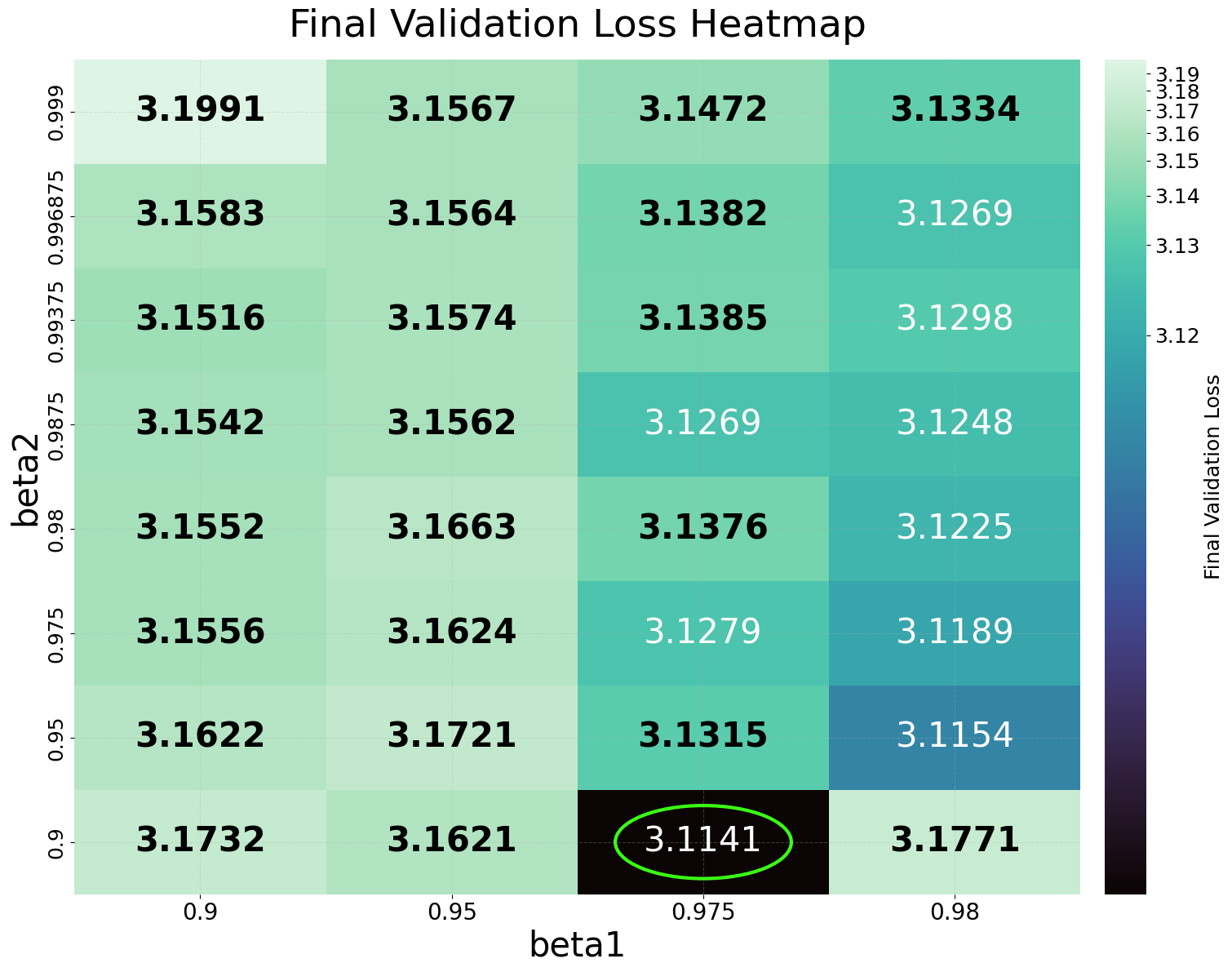}
        \caption{AdamW ($\beta_{1}$ vs $\beta_{2}$)}
    \end{subfigure}\hfill
    \begin{subfigure}{0.33\textwidth}
        \centering
        \includegraphics[width=\linewidth]{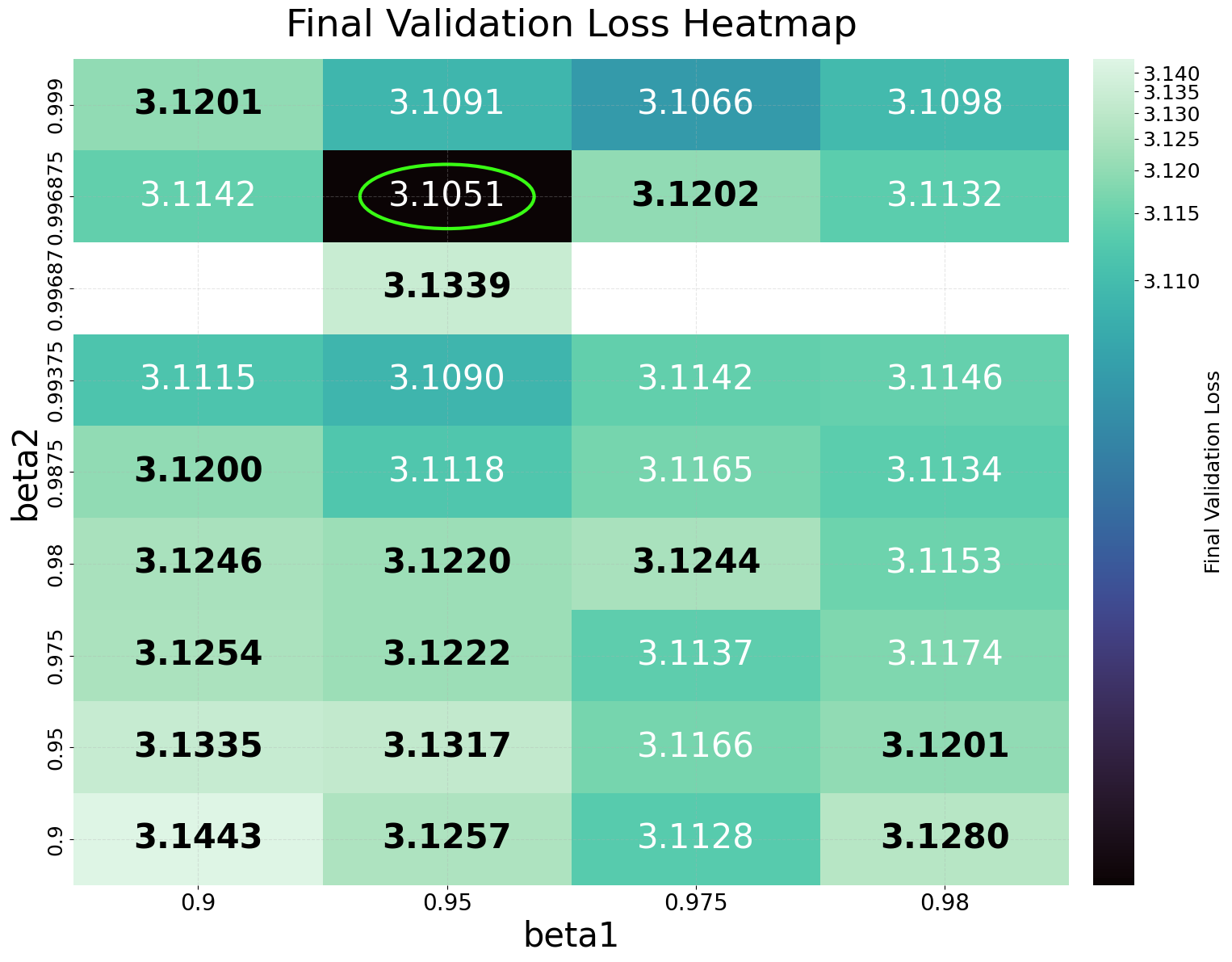}
        \caption{DiLoCo-AdamW ($\beta_{1}$ vs $\beta_{2}$)}
    \end{subfigure}\hfill
    \begin{subfigure}{0.33\textwidth}
        \centering
        \includegraphics[width=\linewidth]{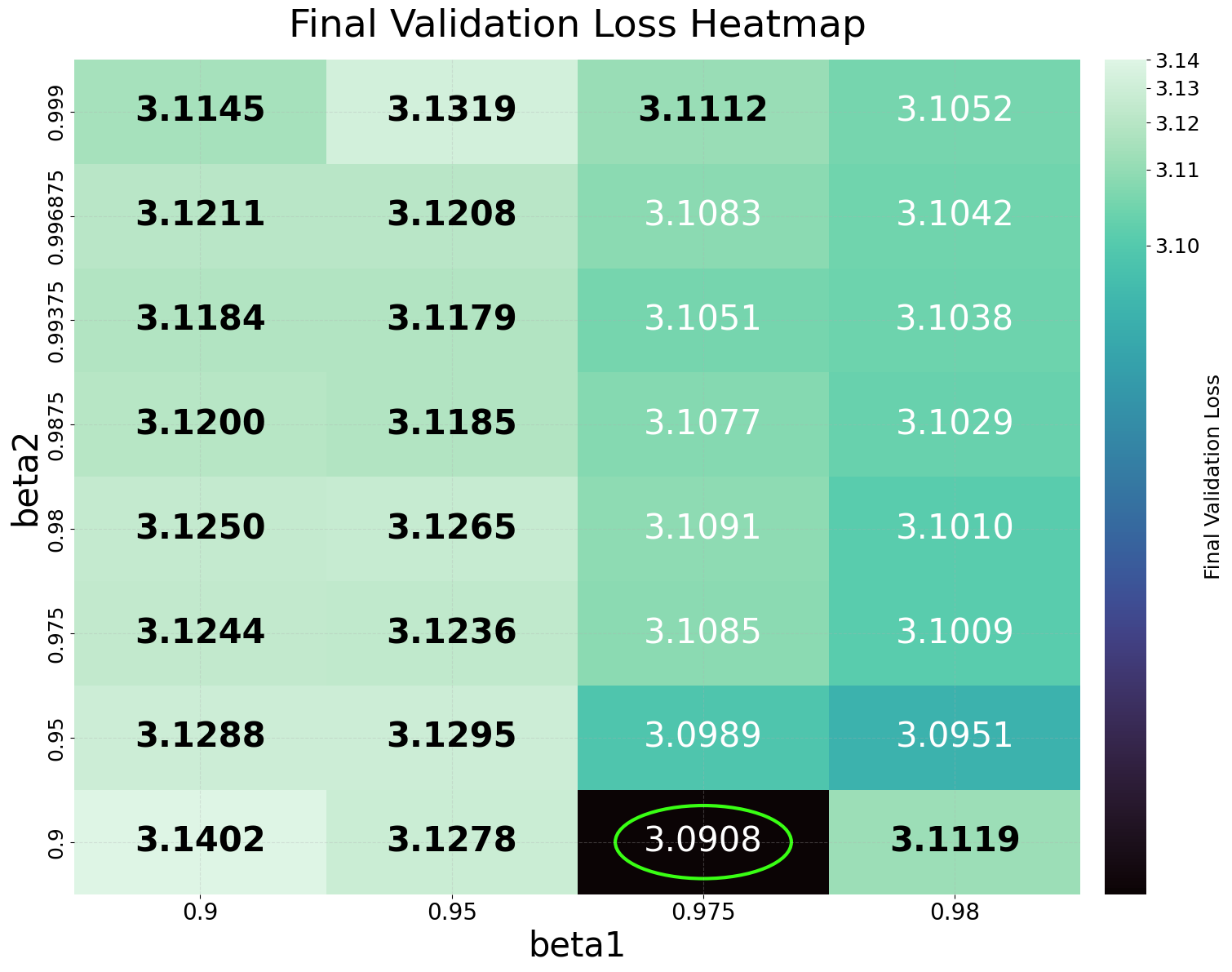}
        \caption{GPA-AdamW ($\beta_{1}$ vs $\beta_{2}$)}
    \end{subfigure}
    \caption{{\bf Comparison of beta hyperparameter sweeps for AdamW, DiLoCo and GPA on Llama-160M model}. The heatmap shows the final validation loss as a function of different values of $\beta_{1}$ and $\beta_{2}$. For clarity, extreme sub-optimal values are shown in black while values within a favorable range are shown in white. The best value is marked with a green circle. During the beta sweeps, the inner learning rates and other hyperparameters are held fixed, so the figures reflect only the correlation between beta hyperparameters. In subsequent phases, we further tune the remaining hyperparameters to determine the final optimal setting for each method.}
    \label{fig:hp_betas}
\end{figure}

\begin{figure}[hbt!]
    \centering
    \begin{subfigure}{0.33\textwidth}
        \centering
        \includegraphics[width=\linewidth]{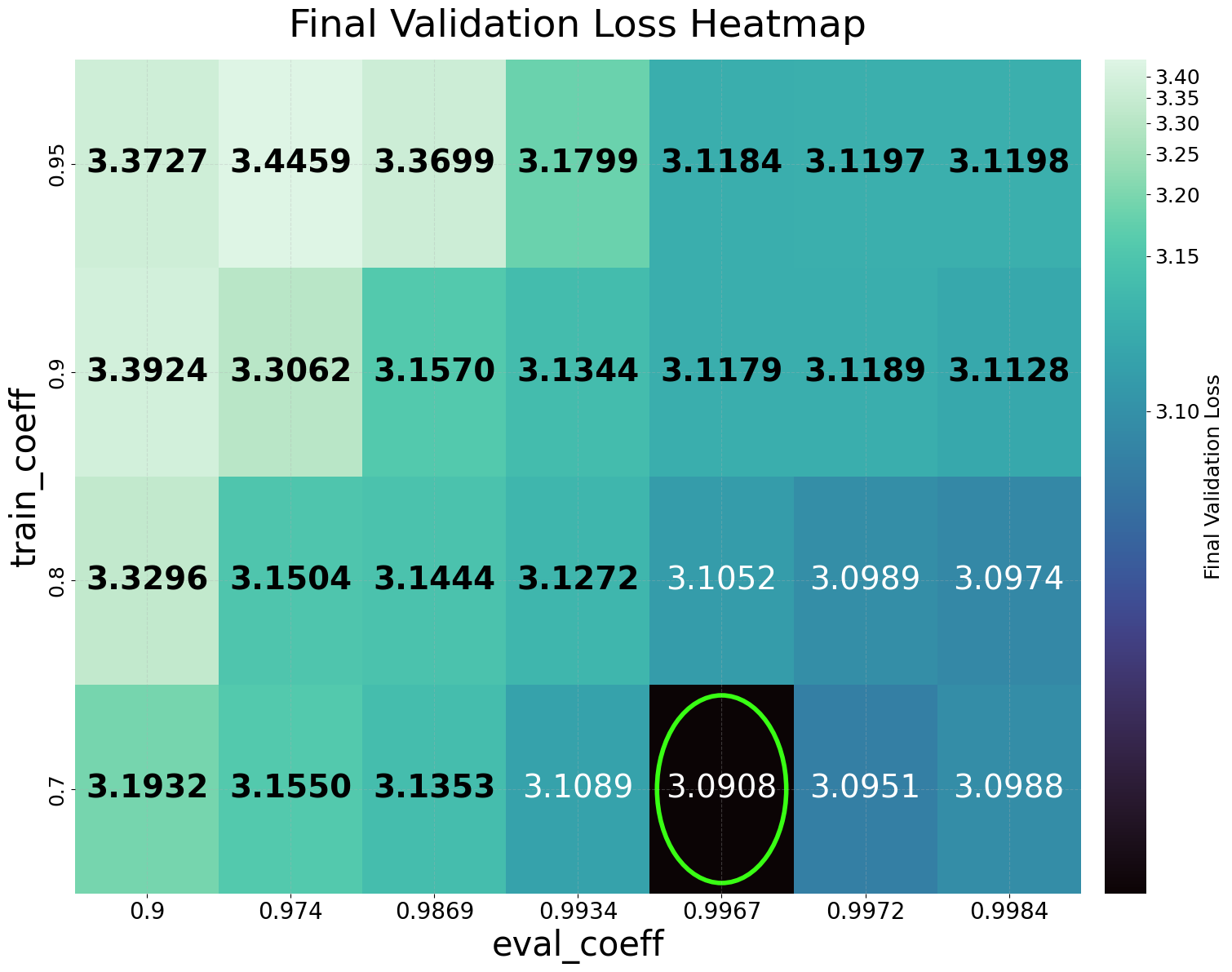}
        \caption{GPA-AdamW ($\mu_{y}$ vs $\mu_{x}$)}
    \end{subfigure}\hfill
    \begin{subfigure}{0.33\textwidth}
        \centering
        \includegraphics[width=\linewidth]{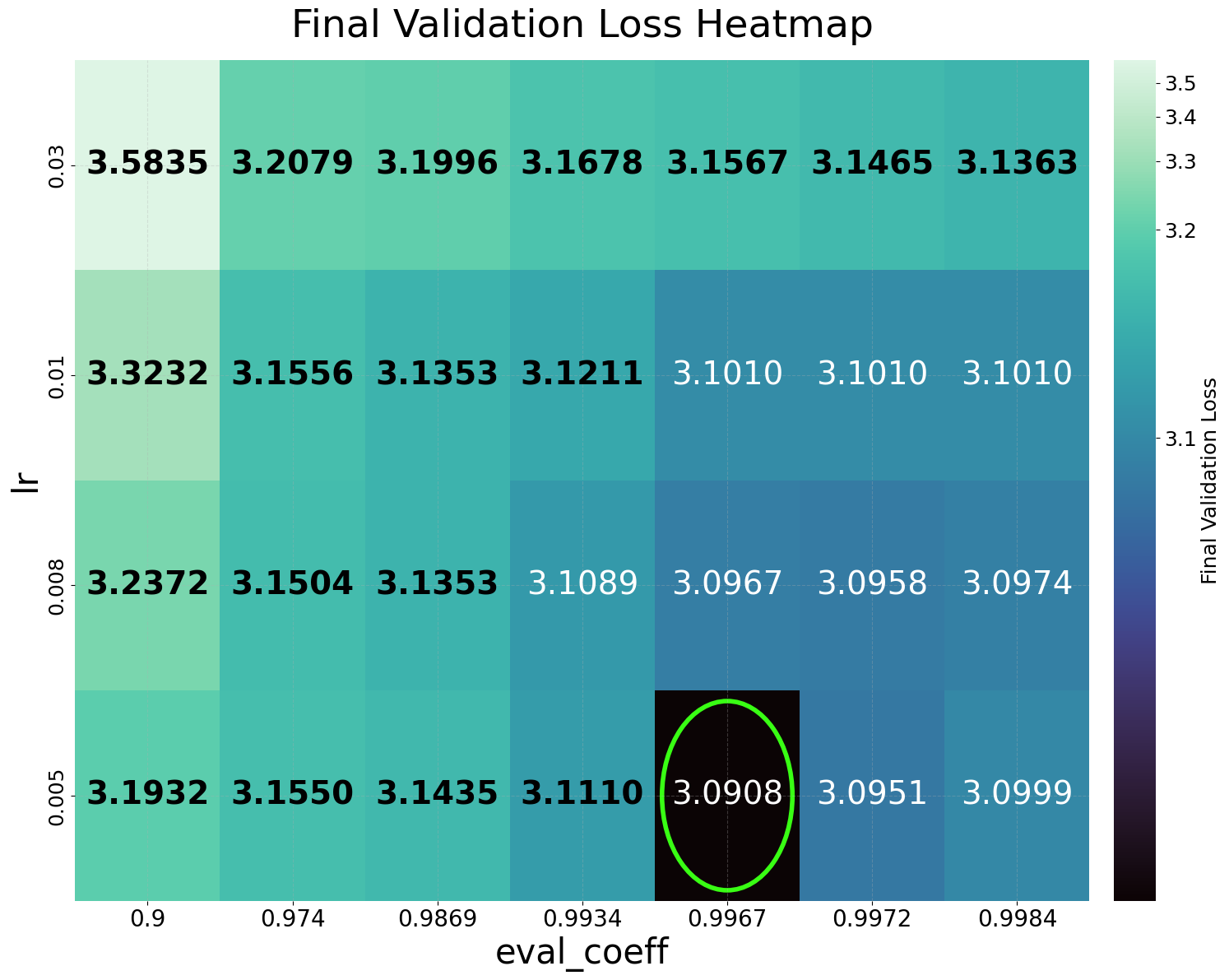}
        \caption{GPA-AdamW ($\gamma$ vs $\mu_{x}$)}
    \end{subfigure}\hfill
    \begin{subfigure}{0.33\textwidth}
        \centering
        \includegraphics[width=\linewidth]{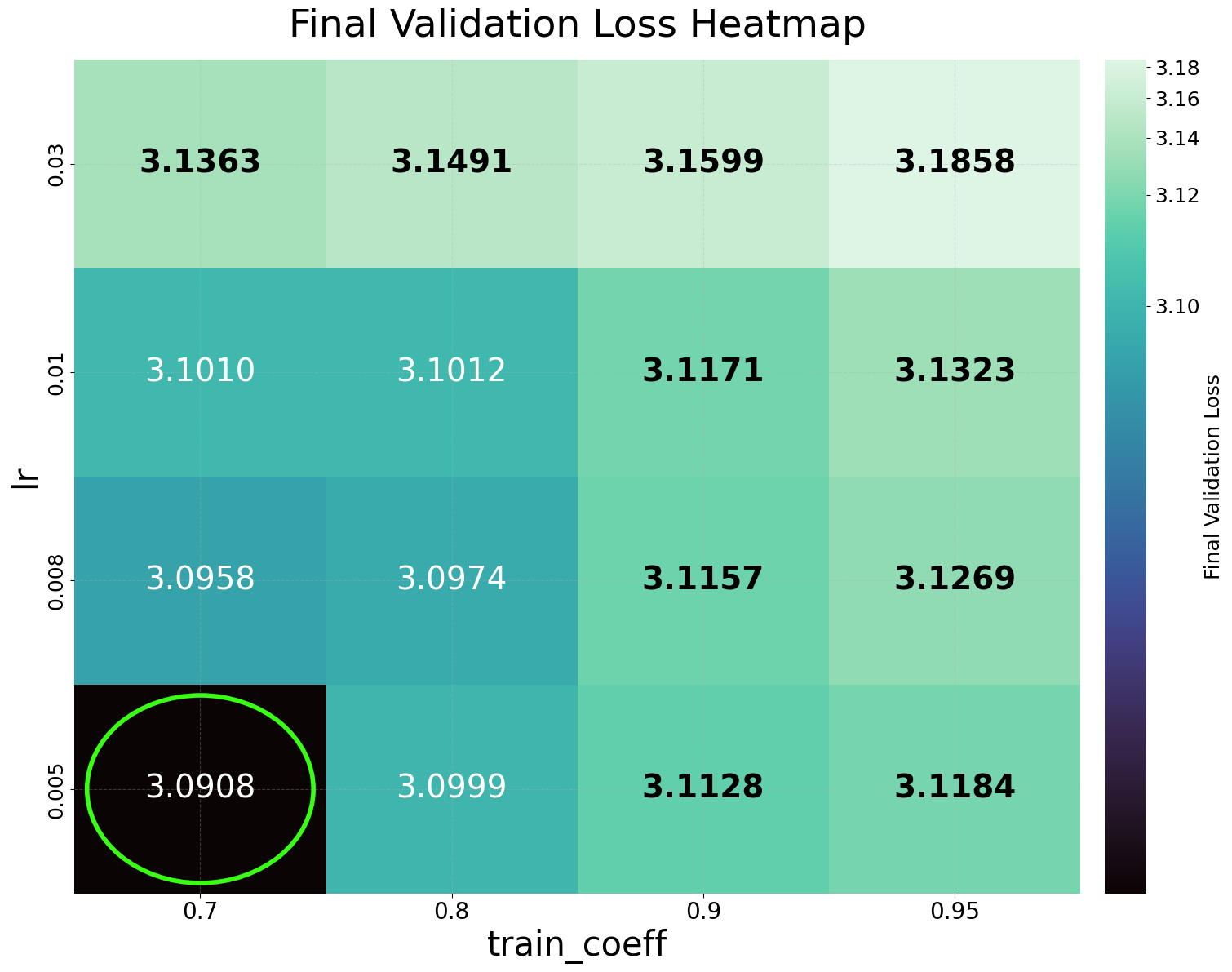}
        \caption{GPA-AdamW ($\gamma$ vs $\mu_{y}$)}
    \end{subfigure}
    \caption{{\bf Analysis of eval and train coefficients $\mu_{x}$ and $\mu_{y}$ for GPA}. The heatmap depicts the final validation loss as a function of different hyperparameters for GPA. Here, \texttt{eval\_coeff} refers to $1 - \mu_x$ and \texttt{train\_coeff} refers to $\mu_y$.}
    \label{fig:hp_gpa_coeffs}
\end{figure}

\begin{figure}[hbt!]
    \centering
    \begin{subfigure}{0.33\textwidth}
        \centering
        \includegraphics[width=\linewidth]{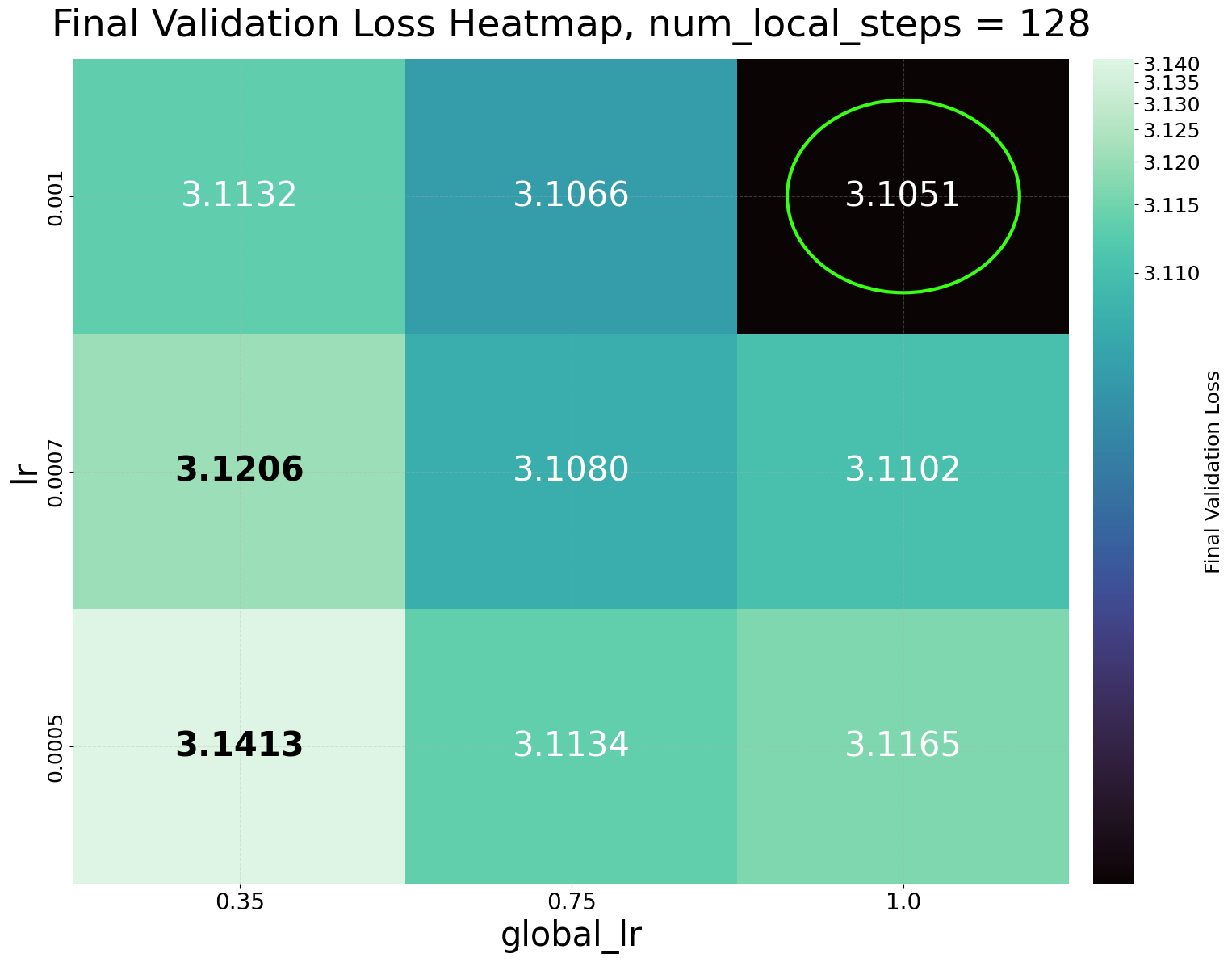}
        \caption{DiLoCo-AdamW ($\gamma$ vs $\tilde{\gamma}$)}
    \end{subfigure}\hfill
    \begin{subfigure}{0.33\textwidth}
        \centering
        \includegraphics[width=\linewidth]{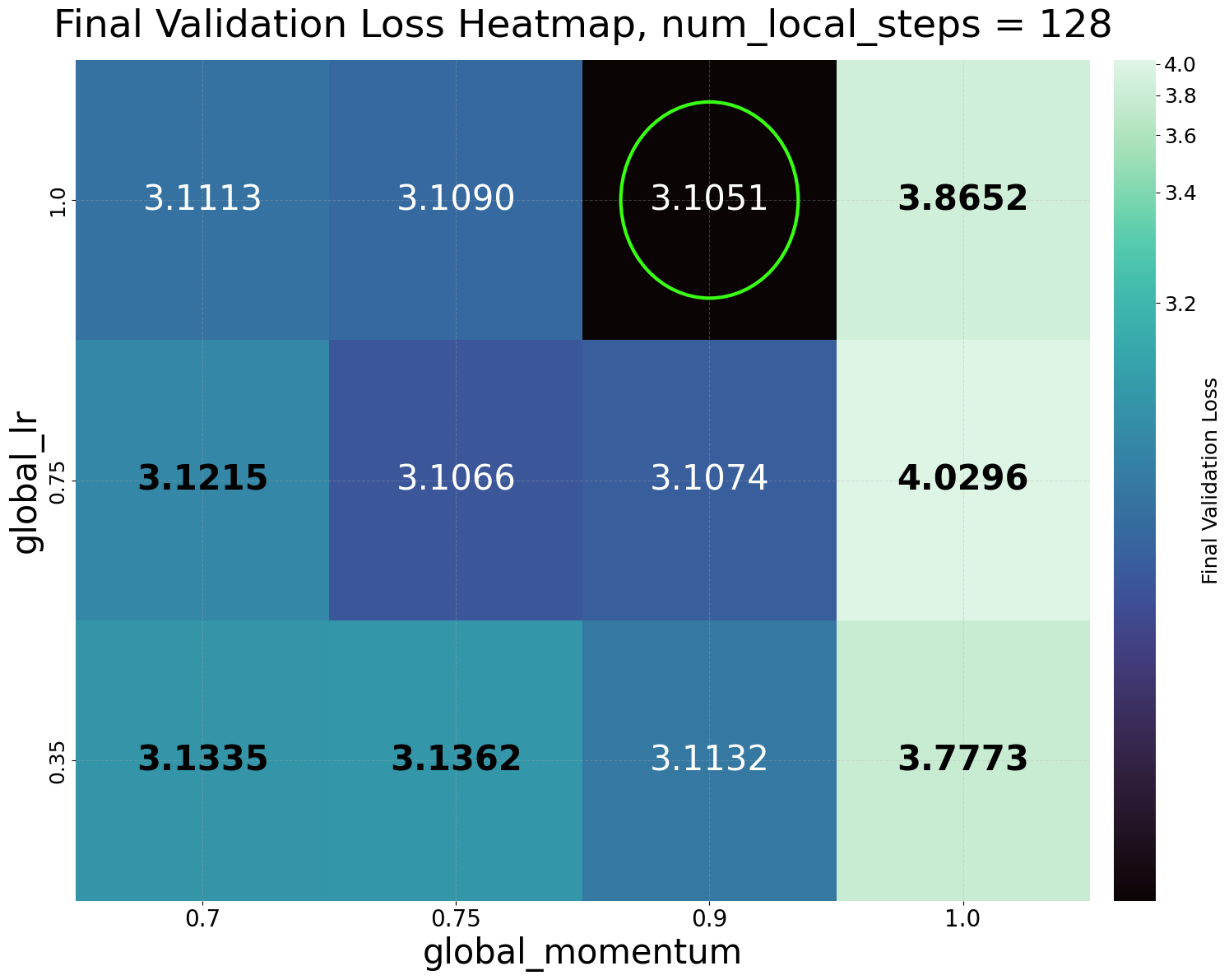}
        \caption{DiLoCo-AdamW ($\tilde{\gamma}$ vs $\mu$)}
    \end{subfigure}\hfill
    \begin{subfigure}{0.33\textwidth}
        \centering
        \includegraphics[width=\linewidth]{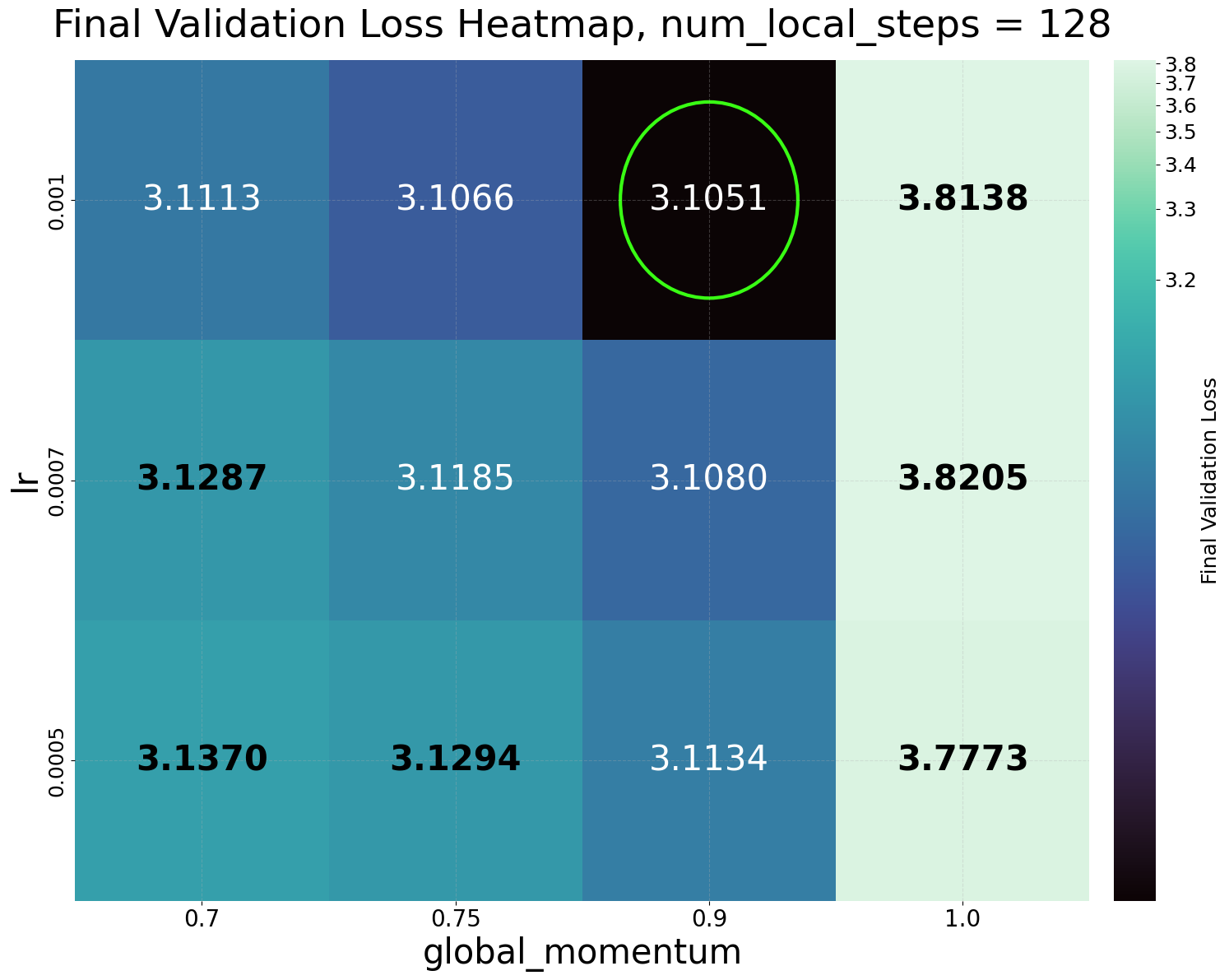}
        \caption{DiLoCo-AdamW ($\gamma$ vs $\mu$)}
    \end{subfigure}
    \caption{{\bf Analysis of DiLoCo's hyperparameters.} The heatmap shows the final validation loss as a function between the inner learning rate $\gamma$, outer learning rate $\tilde{\gamma}$, and momentum hyperparameter $\mu$.}
    \label{fig:hp_diloco}
\end{figure}

{\bf Summary of hyperparameter sweeps.} We summarize the hyperparameter sweeps used in our experiments in Table~\ref{tab:hyperparameter_sweeps}. In Table~\ref{tab:diloco_gpa_mu_x}, we provide a table of conversions from optimal choices of $\mu$ and $H$ in DiLoCo to GPA's choice of $\mu_x$.

\begin{table}[hbt!]
\centering
\caption{Summary of hyperparameter sweeps used in the experiments (Llama-160M). We highlight the best performing hyperparameters in \textbf{bold}.}
\label{tab:hyperparameter_sweeps}

\begin{tabular}{>{\raggedright\arraybackslash}p{3.2cm} 
                >{\columncolor{colorlightgray}}>{\raggedright\arraybackslash}p{3.8cm} 
                >{\columncolor{colorlightgray}}>{\raggedright\arraybackslash}p{3.8cm} 
                >{\columncolor{colorblue}}>{\raggedright\arraybackslash}p{3.8cm}}
\toprule
\textbf{Hyperparameter} & \textbf{AdamW} & \textbf{DiLoCo-AdamW} & \textbf{GPA-AdamW} \\
\midrule
Inner optimizer lr & 7e-4, 3e-3, \textbf{5e-3}, 8e-3, 1e-2 & 5e-4, 7e-4, \textbf{1e-3} & \textbf{5e-3}, 8e-3, 1e-2, 3e-2 \\
Weight decay 
& 0.1  
& 0.1   
& 0.1 \\
Inner Adam $\beta_1$ & 0.9, 0.95, \textbf{0.975}, 0.98 & 0.9, \textbf{0.95}, 0.975, 0.98 & 0.9, 0.95, \textbf{0.975}, 0.98 \\
Inner Adam $\beta_2$ & \textbf{0.9}, 0.95, 0.975, 0.98, 0.9875, 0.99375, 0.996875, 0.999 & 0.9, 0.95, 0.975, 0.98, 0.9875, 0.99375, \textbf{0.996875}, 0.999  & \textbf{0.9}, 0.95, 0.975, 0.98, 0.9875, 0.99375, 0.996875, 0.999 \\
Inner Adam $\epsilon$ & $10^{-8}$ & $10^{-8}$ & $10^{-8}$ \\
Warmup fraction & 10\% & 10\% & 10\% \\
Learning rate schedule & cosine & cosine & cosine \\
Learning rate min fraction & 0.0 & 0.0 & 0.0 \\
GPA coeff $\mu_{y}$ & - & - & \textbf{0.7}, 0.8, 0.9, 0.95 \\
GPA coeff $\mu_{x}$ & - & - & 0.9, 0.9740, 0.9869, 0.9934, 0.9967, 0.9984, 0.9992 \\
Outer optimizer & - & Nesterov & - \\
Outer lr & - & 0.35, \textbf{0.75}, 1.0 & - \\
Outer momentum & - & 0.7. \textbf{0.75}, 0.9, 1.0 & - \\
Communication frequency $H$ & - & 1, 4, 8, 16, 32, 64, 128 & - \\
\bottomrule
\end{tabular}
\end{table}

\begin{table}[hbt!]
    \centering
    \caption{Correspondence between the number of inner steps $H$ and momentum coefficient $\mu_{\diloco}$ in DiLoCo and the momentum coefficient $\mu_x$ in GPA. The values of $\mu_x$ were computed using the expression $\mu_x = \mu_{\diloco}^{1/H}$, with $\mu_{\diloco} = 0.9$ and $H$ as the number of inner steps. For all methods, we use a global batch size of $262$K tokens, sequence length $2048$ and train for $9.6$ billion token for a total of $36622$ steps. We also use a weight decay value of $0.1$.}
    \label{tab:diloco_gpa_mu_x}
    \begin{tabular}{cc}
        \hline
        \textbf{Number of inner steps (DiLoCo)} & \textbf{$\mu_x$ (GPA)} \\
        \hline
        1   & 0.9000 \\
        4   & 0.9740 \\
        8   & 0.9869 \\
        16  & 0.9934 \\
        32  & 0.9967 \\
        64  & 0.9984 \\
        128 & 0.9992 \\
        \hline
    \end{tabular}
\end{table}

\subsection{Hyperparameter Sweeps for Llama-1B}
\label{app:expmt_details_hps_1B}
\noindent {\bf Training setup.} We use the same dataset as in the smaller Llama model, but train longer for 50 billion tokens. To incorporate the larger workload, we utilize four machines (total of 8 GB200 GPUs) for each experiment, with an increased global batch size of 256 sequences with a sequence length of 2048 tokens, resulting in a total batch size of about 524,288 tokens.

{\bf Hyperparameter tuning strategy.} 
\begin{itemize}
    \item For AdamW, we sweep across the top three sets of beta hyperparameters from the smaller model (160M) runs. Because of the model scale, it is computationally prohibitive to carry out a grid search with the same search space used for the 160M model. We set $\epsilon = 10^{-8}$, and sweep the learning rate from $1\cdot 10^{-3}$ through $8\cdot 10^{-3}$.
    \item For DiLoCo-AdamW, we followed the same procedure for beta hyperparameters as for AdamW. We sweep the outer learning rate in $\{0.35, 0.75, 1.0\}$ and the outer momentum in  $\{0.7, 0.75, 0.9, 1.0\}$. We tuned the learning rate in $\{5\cdot 10^{-4}, 1\cdot 10^{-3} \}$. (We found even larger learning rates to be unstable for DiLoCo.) We also sweep through the number of inner steps in $\{16, 32, 64, 128\}$.
    \item For GPA-AdamW, we follow the same guidelines for tuning beta hyperparameters. We sweep $\mu_x$ based on the number of inner steps in DiLoCo (see Table~\ref{tab:diloco_gpa_mu_x}) corresponding to $\{16, 32, 64, 128\}$. We tune $\mu_y$  in $\{0.7, 0.8, 0.9, 0.95\}$ since these were found to be more or less robust values based on several GPA runs. We tuned the learning rate in $\{1\cdot 10^{-3}, 8\cdot 10^{-3}\}$.
\end{itemize}

\subsection{Hyperparameters for Llama-8B}
\label{app:expmt_details_hps_llama_8b}
We use a Llama-3 architecture with 32 layers, 32 attention heads in each, hidden dimension 4K, RoPE positional embeddings, and context length of 8K. The total number of parameters is 8 billion. The learning rate is warmed up over 2,000 iterations and then annealed to $0.001 \times$ of the peak value. We train with a batch size of 4 million tokens over 24,000 steps, with 100 billion total number of tokens processed. The data consists primarily of \texttt{dclm-baseline-1.0} and \texttt{github} code.

Due to the high cost of each run, we did a preliminary hyperparameter search with a version of the model downsized to 250M parameters, and then did a more limited search on the full 8B model. The explored combinations of hyperparameters include learning rates of 0.001, 0.003, 0.005, 0.007, and weight decay values of 0.07, 0.1, 0.15, and 0.2. We used $\beta_1=0.9, \beta_2=0.95$, $\epsilon=10^{-8}$ and gradient norm clipping at value 1. Both best performing variant of AdamW and GPA-AdamW use weight decay 0.1, and learning rate 0.005. GPA-AdamW performed similarly with any $\mu_y\in\{0.7, 0.8, 0.9\}$, giving final loss values 1.8645, 1.8655, 1.8662, which are all better than AdamW's loss value 1.8732.

\subsection{Hyperparameter Sweeps for ViT ImageNet Experiments}
\label{app:expmt_details_hps_vit}
For data augmentations, we use RandAugment with strategy ``rand-m15-n2'', cutmix $\alpha=1$, mixup with probability $0.5$ and $\alpha=0.8$, no dropout, and no label smoothing. This setup has been reported to provide high validation accuracy values. For privacy reasons, we use the version of ImageNet-1k with faces blurred.

We test values of $\mu_y$ from $\{0.3, 0.5, 0.7, 0.8, 0.9\}$. The difference between the resulting curves is less than $0.5\%$ and we found that the values of 0.7, 0.8, and 0.9 gave similar performance close to optimal. We provide the results with identical hyperparameters (weight decay 0.1, $\beta_1=0.8$, $\beta_2=0.999$) but varying values of $\mu_y$ and learning rates in Figure~\ref{fig:compare_muy}, alongside the results for AdamW with weight decay 0.1, $\beta_1=0.9$, $\beta_2=0.999$. All numbers are obtained by averaging over 12 random seeds.

\begin{figure}[hbt!]
    \centering
    \includegraphics[width=0.48\linewidth]{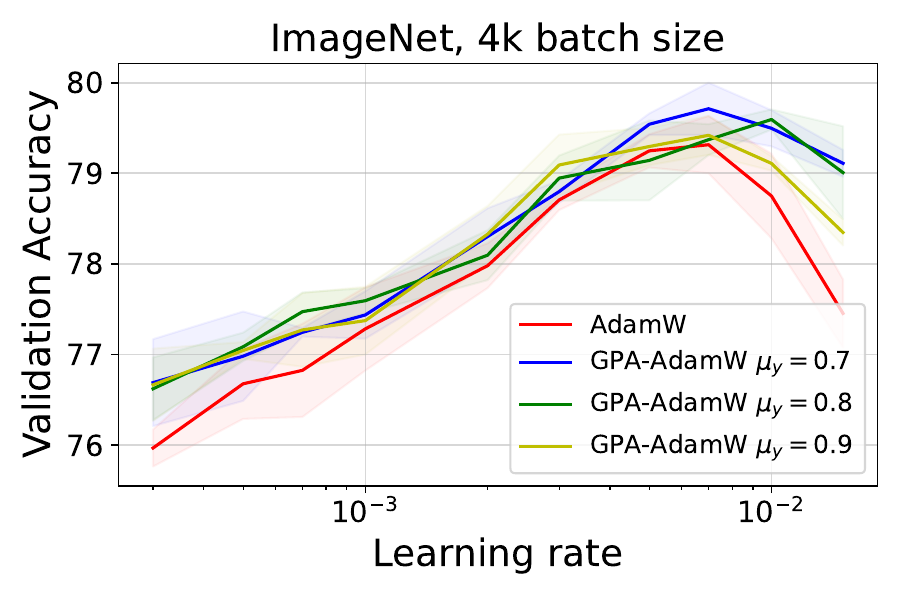}
    \caption{Comparison of GPA-AdamW with varying $\mu_y$ and AdamW on ImageNet ViT-S/16 from \texttt{timm} with data augmentations using a {\bf batch size of 4,096 samples}.}
    \label{fig:compare_muy}
\end{figure}

As a reminder, all methods use linear learning rate warmup over the first 5 epochs, and gradient clipping at norm 1. 

\begin{table}[hbt!]
\centering
\caption{Summary of hyperparameter sweeps used in the ImageNet experiment with \textbf{batch size 4,096}. We highlight the best performing hyperparameters in \textbf{bold}.}
\label{tab:hyperparameter_sweeps}

\begin{tabular}{>{\raggedright\arraybackslash}p{3.15cm} 
                >{\columncolor{colorlightgray}}>{\raggedright\arraybackslash}p{2.95cm} 
                >{\columncolor{colorlightgray}}>{\raggedright\arraybackslash}p{2.95cm} 
                >{\columncolor{colorlightgray}}>{\raggedright\arraybackslash}p{2.95cm}
                >{\columncolor{colorblue}}>{\raggedright\arraybackslash}p{2.95cm}}
\toprule
\textbf{Hyperparameter} & \textbf{AdamW} & \textbf{DiLoCo-AdamW} & \textbf{Schedule-Free-AdamW} & \textbf{GPA-AdamW} \\
\midrule
Inner optimizer lr
& 0.002, 0.003, \textbf{0.005}, 0.007, 0.01 
& 0.0005, 0.0007, 0.001, \textbf{0.0015}, 0.002 
& 0.005, 0.007, \textbf{0.01}, 0.015, 0.02
& 0.003, 0.005, \textbf{0.007}, 0.01, 0.015 \\
Weight decay
& 0.07, \textbf{0.1}, 0.15, 0.2 
& 0.01, 0.02, 0.03, \textbf{0.05}, 0.07, 0.1, 0.15 
& 0.02, 0.03, 0.05, 0.07, \textbf{0.1}, 0.15
& 0.03, 0.05, 0.07, \textbf{0.1}, 0.15, 0.2 \\
Inner Adam $\beta_1$ & 0.8, \textbf{0.9} & \textbf{0.8}, 0.9 & \textbf{0.8}, 0.9 & \textbf{0.8}, 0.9 \\
Inner Adam $\beta_2$ & 0.999 & 0.999  & 0.999 & 0.999 \\
Inner Adam $\epsilon$ & $10^{-8}$ & $10^{-8}$ & $10^{-8}$ & $10^{-8}$ \\
Learning rate schedule & cosine & cosine & no schedule & cosine \\
GPA coeff $\mu_{y}$ / SF $\beta$ & - & - & 0.5, 0.8, 0.9, \textbf{0.95}, 0.98 & 0.3, 0.5, \textbf{0.7}, 0.8, 0.9 \\
GPA coeff $\mu_{x}$ & - & - & - & 0.9934 \\
Outer optimizer & - & Nesterov & - & - \\
Outer lr & - & 0.7 & - & - \\
Outer momentum & - & 0.9 & - & - \\
Communication frequency $H$ & - & 16 & - & - \\
\bottomrule
\end{tabular}
\end{table}

\begin{table}[hbt!]
\centering
\caption{Summary of hyperparameter sweeps used in the ImageNet experiment with \textbf{batch size 16,384}. We highlight the best performing hyperparameters in \textbf{bold}.}
\label{tab:hyperparameter_sweeps}

\begin{tabular}{>{\raggedright\arraybackslash}p{3.15cm} 
                >{\columncolor{colorlightgray}}>{\raggedright\arraybackslash}p{2.95cm} 
                >{\columncolor{colorlightgray}}>{\raggedright\arraybackslash}p{2.95cm} 
                >{\columncolor{colorlightgray}}>{\raggedright\arraybackslash}p{2.95cm} 
                >{\columncolor{colorblue}}>{\raggedright\arraybackslash}p{2.95cm}}
\toprule
\textbf{Hyperparameter} & \textbf{AdamW} & \textbf{DiLoCo-AdamW} & \textbf{Schedule-Free-AdamW} & \textbf{GPA-AdamW} \\
\midrule
Inner optimizer lr 
& 0.005, 0.007, 0.01, \textbf{0.015}, 0.02 
& 0.001, 0.0015, 0.002, \textbf{0.003}, 0.005, 0.007
& 0.01, \textbf{0.015}, 0.02
& 0.005, 0.007, 0.01, 0.015, \textbf{0.02}, 0.03 \\
Weight decay 
& 0.03, 0.05, \textbf{0.07}, 0.1, 0.15, 0.2
& 0.07, \textbf{0.1}, 0.15, 0.2 
& 0.07, \textbf{0.1}, 0.15
& 0.07, \textbf{0.1}, 0.15, 0.2 \\
Inner Adam $\beta_1$ & 0.8, \textbf{0.9} & \textbf{0.8}, 0.9 & \textbf{0.8}, 0.9 & \textbf{0.8}, 0.9 \\
Inner Adam $\beta_2$ & 0.999 & 0.999  & 0.999 & 0.999 \\
Inner Adam $\epsilon$ & $10^{-8}$ & $10^{-8}$ & $10^{-8}$ & $10^{-8}$ \\
Learning rate schedule & cosine & cosine & no schedule & cosine \\
GPA coeff $\mu_{y}$ / SF $\beta$ & - & - & 0.7, \textbf{0.8}, 0.9, 0.95 & 0.5, 0.7, 0.8, \textbf{0.9} \\
GPA coeff $\mu_{x}$ & - & - & - & 0.9934 \\
Outer optimizer & - & Nesterov & - & - \\
Outer lr & - & 0.7 & - & - \\
Outer momentum & - & 0.9 & - & - \\
Communication frequency $H$ & - & 16 & - & - \\
\bottomrule
\end{tabular}
\end{table}

\end{document}